 \documentclass[accepted]{uai2023} % after acceptance, for a revised
\usepackage[american]{babel}
\usepackage[linesnumbered,ruled,vlined]{algorithm2e}
\usepackage{relsize}
\usepackage{float}
\usepackage{multirow}
\usepackage{array}
\usepackage{caption}
\usepackage{lipsum}

\usepackage{natbib} % has a nice set of citation styles and commands
\usepackage{mathtools} % amsmath with fixes and additions
\usepackage{booktabs} % commands to create good-looking tables
\usepackage{tikz} % nice language for creating drawings and diagrams
\usepackage{amsmath}
\usepackage{amssymb}
\usepackage{mathtools}
\usepackage{amsthm}
\usepackage[capitalize,noabbrev]{cleveref}
\theoremstyle{plain}
\newtheorem{theorem}{Theorem}[section]

\newtheorem{lemma}[theorem]{Lemma}

\theoremstyle{definition}
\newtheorem{definition}[theorem]{Definition}

\theoremstyle{remark}

\newtheorem*{theorem1}{\textbf{Theorem~\ref{thm:lb_epehe}}}
\newtheorem*{theorem2}{\textbf{Theorem~\ref{thm:smiple_d1}}}
\newtheorem*{theorem3}{\textbf{Theorem~\ref{thm:simple_ipm}}}
\newtheorem*{lemma2}{\textbf{Lemma~\ref{lemma2}}}
\newtheorem*{theorem5}{\textbf{Theorem~\ref{main_theorem_1}}}

 \newcommand{\ind}{\perp\!\!\!\!\perp} 
\title{Transfer Learning for Individual Treatment Effect Estimation}

\author{Ahmed Aloui $^*$}
\author{Juncheng Dong $^*$}
\author{Cat P. Le}
\author{Vahid Tarokh}
\affil{
Department of Electrical and Computer Engineering, Duke University
}

\begin{document}
\maketitle

\def\thefootnote{*}\footnotetext{Equal Contribution.}

\begin{abstract}
This work considers the problem of transferring causal knowledge between tasks for Individual Treatment Effect (ITE) estimation. To this end, we theoretically assess the feasibility of transferring ITE knowledge and present a practical framework for efficient transfer. A lower bound is introduced on the ITE error of the target task to demonstrate that ITE knowledge transfer is challenging due to the absence of counterfactual information. Nevertheless, we establish generalization upper bounds on the counterfactual loss and ITE error of the target task, demonstrating the feasibility of ITE knowledge transfer. Subsequently, we introduce a framework with a new Causal Inference Task Affinity (CITA) measure for ITE knowledge transfer. Specifically, we use CITA to find the closest source task to the target task and utilize it for ITE knowledge transfer. Empirical studies are provided, demonstrating the efficacy of the proposed method. We observe that ITE knowledge transfer can significantly (up to 95\%) reduce the amount of data required for ITE estimation.
\end{abstract}

%Recent developments in deep representation models have resulted in promising approaches for estimating Individual Treatment Effects (ITE). This paper examines the theoretical and empirical properties of transferring causal knowledge from one task to another with limited data. We provide bounds on the target task's counterfactual loss and ITE error, indicating the transferability of causal knowledge. We observe that the absolute values of ITEs are invariant under the symmetric group's action on the treatments' labels. Given this invariance, we propose a symmetrized task affinity for calculating the similarity between tasks. Empirical studies on various datasets show that this proposed task affinity strongly correlates with the counterfactual loss and that transferring causal knowledge can significantly reduce the required data by up to 95\% compared to training from scratch.
%\A{Maybe it is better to stick with one name for the task affinity: either "Symmetrized task affinity", "Symmetrized Task Affinity" or "Label Invariant task affinity", If Symmetrized may cause confusion with the mathematical definition of symmetry for distance then we stick with label invariant. }
%\JD{We can try to think of a different word to "Symmetrized". "Label-invariant" also doesn't sound like the right word. We are trying to say that causal inference tasks have a symmetry property and our proposed task affinity captures it. Or we can simply give it a name like CITA (causal inference task affinity).}
%\A{task affinity: ITE task affinity}

\section{Introduction}
% \blfootnote{* Equal Contribution. Alphabetical Order.}
% One of the most remarkable skills of humans is their capability to apply causal knowledge acquired from one scenario to \textit{similar} scenarios. This ability is highly desirable in neural networks due to its wide range of potential applications.
Assessing the effects of treatments on people (i.e., the \emph{Individual Treatment Effect} (ITE) estimation)  is of significant interest to various research communities, such as those studying medicine and social policy making. In order to study the causal relationship between the outcome and the treatment, however, researchers must gather sufficient data samples from randomized control trials. This process can be both costly and time-consuming~\citep{vaccine}. To this end, it is desirable to utilize knowledge from different but closely related problems with \emph{transfer learning}. For instance, new vaccines must be developed for treatment when the viruses undergo mutation. Suppose the mutated viruses can be related to the known ones by a similarity measure. In that case, the effects of vaccine candidates can be quickly estimated based on this similarity with a small amount of data collected from the new scenario. Hence, this approach can notably accelerate the study.

% For instance, new vaccines must be developed for treatment when viruses undergo mutation. To evaluate the effectiveness of these new vaccine candidates, researchers must gather sufficient data samples from randomized control trials, which can be both costly and time-consuming~\citep{vaccine}.
% Suppose the mutated viruses can be related to known ones by a similarity measure. In that case, the effects of vaccine candidates can be quickly calculated based on this similarity with a small amount of data collected from the new scenario. In other words, transfer learning methods can accelerate the study of the effects of various treatments, e.g., applications in medicines, personal training, and social policy~\citep{medical_survey}.

While the recent progress in transfer learning is very promising~\citep{TLCVsurvey, TLNLPsurvey, TLsurvey2008, TLsurvey2021}, a major challenge for transferring causal knowledge arises from non-causal (spurious) correlations to which the statistical learning models are vulnerable. For example, a classifier may learn to use the background colors to differentiate images of camels and horses, as these objects are frequently depicted against different colored backgrounds~\citep{Arjovsky2019, geirhos2018, Beery2018}. In practice, the performance of the ITE estimation models can \emph{never} be evaluated because the counterfactual data is inaccessible, as shown in Figure~\ref{fig:causal_inference}. This problem is known in the literature as \textit{the fundamental problem of causal inference}~\citep{rubin1974,holland1986}. For instance, to compute the effect of vaccination on a person at some given time, that individual must both be administered the vaccine, and also remain unvaccinated, which is obviously absurd. This scenario is very different from the conventional supervised learning problems, where researchers often use a separate validation set in order to estimate the accuracy of the trained model.

The aforementioned challenge implies that much attention must be paid to selecting the appropriate source model in causal knowledge transfer. Additionally, similar scenarios to the given target task must be determined using a distance accounting for the \textit{immeasurable} counterfactual losses in scenarios under consideration. 
% In this work, we first analyze the viability of causal knowledge transfer by presenting a set of generalization bounds for transfer learning between causal inference tasks.
In this work, we first present a lower bound and a set of generalization bounds for transfer learning between causal inference tasks in order to demonstrate both the difficulty and viability of causal knowledge transfer. 
While these theoretical bounds are informative, a method is needed for selecting the optimal source model from multiple source tasks. This is discussed in Section~\ref{sec:task_distance}, where we introduce a framework endowed with a new task affinity, namely the Causal Inference Task Affinity (CITA), tailored explicitly for causal  knowledge transfer. This task affinity is used for selecting the ``closest'' source task. Subsequently its knowledge (e.g., trained models, source dataset) is utilized in the learning of the target task, as depicted in Figure~\ref{fig:intro_plot_1}. Our contributions are summarized below:

\begin{figure}[t]
\centering
\includegraphics[width=0.47\textwidth]{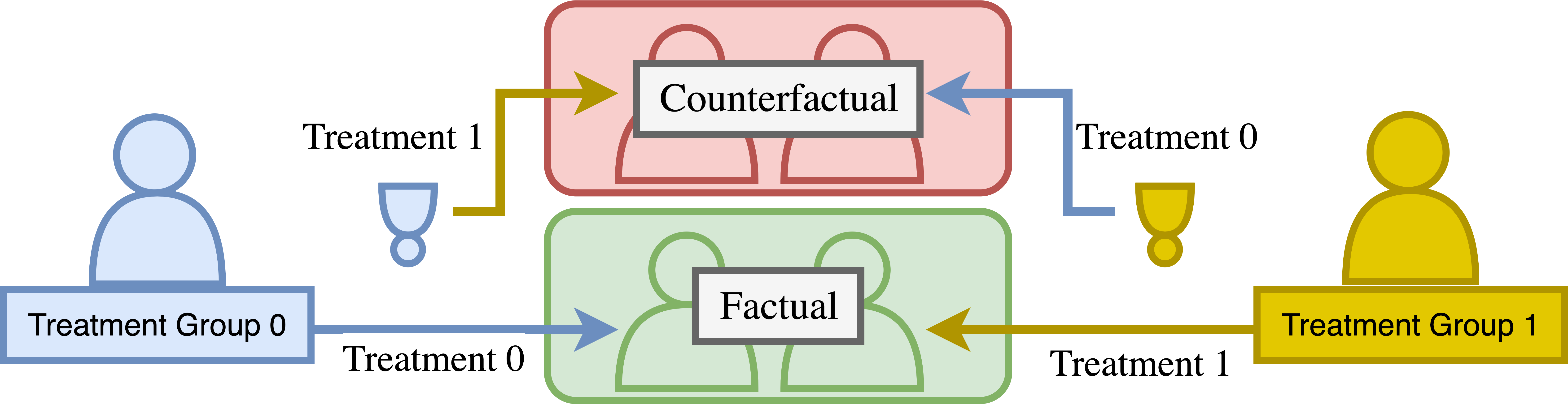}
\caption{Inaccessibility to counterfactual data (e.g., a parallel  universe where the treatments are reversed) makes transferring causal knowledge more challenging.}
\label{fig:causal_inference}
\end{figure}

\begin{enumerate}
\item We establish a new lower bound to demonstrate the challenges of transferring ITE knowledge. Additionally, we prove new regret bounds for learning the counterfactual outcomes and ITEs of the target tasks in causal transfer learning scenarios. These bounds demonstrate the feasibility of transferring ITE knowledge by stating that the error of any source model on the target task is upper bounded by quantifiable measures related to \textit{(i)} the performance of the source model on the source task and \textit{(ii)} the \textit{differences} between the source and the target causal inference tasks.
% These bounds prove the feasibility of transferring causal knowledge. 

\item We introduce CITA, a task affinity for causal inference, which captures the symmetry of ITEs  (i.e., invariance to the relabeling of treatment assignments under the action of the symmetric group). Additionally, we provide theoretical (e.g., Theorem F.3) and empirical evidence to show that CITA is highly correlated with the counterfactuals loss, which is \textit{not measurable} in practice.

\item We propose an ITE estimation framework and a set of causal inference datasets \textit{suitable for learning causal knowledge transfer}. The empirical evidence on the above datasets demonstrates that our methods can estimate the ITEs for the target task with significantly fewer (up to 95\% reduction) data samples compared to the case where transfer learning is not performed.

% Some of these are well-established datasets in the literature, while others are derived from known causal relations in social sciences, physics, and mathematics. 
% \item We provide empirical evidence based on the above datasets that our methods can compute the ITEs for the target task with significantly fewer (up to 95\% reduction) data points compared to the case where transfer learning is not performed. 
\end{enumerate}

\section{Related Work}
\label{sec:related_work}
% In the setting of transfer learning~\citep{TLsurvey2008, TLsurvey2021}, prior learned models are used to increase the learning efficiency and decrease the required data. For instance, the parameters from a trained model may be used as initialization values for the target task. 
Many approaches in transfer learning
% ~\citep{Thrun1998, Blum1998, silver2008guest, mihalkova2007mapping, niculescu2007inductive, luolabel, Razavian:2014:CFO:2679599.2679731, 5288526, DBLP:journals/corr/abs-1801-06519, finn2016deep, DBLP:journals/corr/FernandoBBZHRPW17, DBLP:journals/corr/RusuRDSKKPH16, zamir2018taskonomy, kirkpatrick2017overcoming, chen2018coupled} 
~\citep{Thrun1998, Blum1998, silver2008guest, Razavian:2014:CFO:2679599.2679731, finn2016deep, DBLP:journals/corr/FernandoBBZHRPW17, DBLP:journals/corr/RusuRDSKKPH16, 9054118} have been proposed, analyzed and applied in various machine learning applications. Transfer learning techniques inherently assume that prior knowledge in the selected source model helps with learning a target task~\citep{TLsurvey2008, TLsurvey2021}. In other words, these methods often do not consider the selection of the base task to perform knowledge transfer. Consequently, in some rare cases, transfer learning may even degrade the performance of the model~\cite{pmlr-v119-standley20a}. In order to avoid potential performance loss during knowledge transfer to a target task, \textit{task affinity} (or task similarity) is considered as a selection method that identifies a group of closest base candidates from the set of the prior learned tasks.  
Task affinity has been investigated and applied to various domains (e.g., transfer learning~\citep{zamir2018taskonomy, dwivedi2019, wang2019neural}, neural architecture search~\citep{le2021task, 9766163, 2021arXiv210300241L}, few-shot learning~\citep{pal2019zeroshot, le2022task}, multi-task learning~\citep{pmlr-v119-standley20a}, continual learning~\citep{kirkpatrick2017overcoming, chen2018coupled}). 
%\JD{The following sentences can be removed if in need of space.}
%The related prior learned tasks are identified with similarity measures and then employed for knowledge transfer. 
%\JD{The next two sentences should be moved above.}\A{I think we can delete this}
%Task affinity is inherently a non-commutative measure as it may be straightforward to transfer the knowledge from a more comprehensive task to a  simpler task than the other way around ~\citep{le2021task}. 

While transfer learning and task affinity have been investigated in numerous application areas, their applications to causal inference have yet to be thoroughly investigated.
Neyman-Rubin Causal Model~\citep{neyman1923,donald2005causal} and Pearl's Do-calculus~\citep{pearl} are popular frameworks for causal studies based on different perspectives. A central question in the Neyman-Rubin Causal Model framework is determining conditions for identifiability of causal quantities such as \textit{Average} and \textit{Individual Treatment Effects}. Previous work considered estimators for Average Treatment Effect based on various methods such as Covariate Adjustment \citep{rubin}, weighting methods such as those utilizing propensity scores \citep{propensity}, and Doubly Robust estimators \citep{doubly}. With the emergence of Machine Learning techniques, more recent approaches to causal inference include the applications of decision trees\citep{athey,athey2016recursive}, Gaussian Processes \citep{alaa}, and Generative Modeling \citep{ganite} to ITE estimation. In particular, deep neural networks have successfully learned ITEs and estimated counterfactual outcomes by data balancing in the latent domain \citep{johanson, shalit}. Please note that the transportation of causal graphs is another well-studied closely related field in the causality literature \citep{Bareinboim2021}. It studies transferring knowledge of causal relationships in Pearl's do-calculus framework. In contrast, in this paper, we are interested in transferring knowledge of ITE from a source task to a target task in the Neyman-Rubin framework using representation learning.
A closely related problem to ours is the domain adaptation problem for ITE estimation, as explored in ~\citep{bica2022transfer,vo2022adaptive,aglietti2020multi}. These works primarily focus on situations where only the distribution of populations changes, leaving the causal functions unaltered. In our research, we provide theoretical analysis and empirical studies for the case where both the population distributions and the causal mechanisms can change.

\section{Mathematical Background}\label{sec:math_bg}
%We first establish the notation and briefly review the required mathematical background for causal inference.

\subsection{Causal Inference}
\label{causal}
%\textbf{You need to say what covariate and treatment variables are or refer to some paper}
 
Let $X \in \mathcal{X}\subset\mathbb{R}^{d}$ be the covariates (i.e., input features), $A \in \{0,\ldots,M\}$ be the treatment, and $Y \in \mathcal{Y} \subset \mathbb{R}$ be the factual (observed) outcome.  For every $j\in \{0,\ldots,M\}$ we define $Y_{j}$ to be the \textit{potential outcome} \citep{rubin1974} that would have been observed if only the treatment $A=j, \, j \in \{0,1, \cdots, M\}$ was assigned. 
%example
In the medical context, for instance, $X$ is the individual information (e.g., weight, heart rate), $A$ is the treatment assignment (e.g., $A=0$ if the individual did not receive a vaccine, and $A=1$ if the individual is vaccinated), $Y$ is the outcome (e.g., mortality data). A \textit{causal inference dataset} is a collection of factual observations $D_{F} = \{(x_{i},a_{i}),y_{i}\}_{i=1}^{N}$, where $N$ is the number of samples. We assume these samples are independently drawn from the same factual distribution $p_F$. In a parallel universe, if the roles of the treatment and control groups were reversed, we would have observed a different set of samples $D_{CF}$ sampled from the counterfactual distribution $p_{CF}$.
In this work, we present our results for the binary case, i.e., $M=1$. However, our approach can be easily extended to any positive integer $M < \infty$. In the binary case, the individuals who received treatments $A=0$ and $A=1$ are respectively denoted by the control and treatment groups.

\begin{definition}[ITE]
The Individual Treatment Effect (ITE), referred to as the Conditional Average Treatment Effect (CATE) \citep{imbensbook}, is defined as:
\begin{align}
\label{ITE}
    \forall x \in \mathcal{X}, \; \tau(x) = \mathbb{E}[{Y_{1}-Y_{0}|X=x}]    
\end{align}
\end{definition}

We assume that the data generation process respects the \textit{overlap}, i.e. $\forall x \in \mathcal{X}, 0<p(a=1|x)<1$, and \textit{conditional unconfoundedness}, i.e. $(Y^{1}, Y^{0}) \perp \!\!\! \perp A |X$ \citep{robins1987}. These assumptions are sufficient conditions for the ITE to be identifiable \citep{imbens}. We also assume that the true causal relationship is described by a function $f(x,a)$, which can be expressed as an expected value in the non-deterministic case. 
By definition $\tau(x)=f(x,1)-f(x,0)$. Let $\hat f(x,a)$ denote a hypothesis that estimates the true function $f(x,a)$. Thus, the ITE function can then be estimated as $\hat{\tau}(x) = \hat{f}(x,1) - \hat{f}(x,0)$. We use $\ell_{\hat{f}}(x,a,y)$ to denote a loss function that quantifies the performance of $\hat{f}(\cdot, \cdot)$. A possible example is the $L^2$ loss defined as $\ell_{\hat{f}}(x,a,y)=(y-\hat{f}(x,a))^2$.

\begin{definition}[Factual Loss]
For a hypothesis $\hat{f}$ and a loss function $l_{\hat{f}}$, the factual loss is defined as:
\begin{equation}
\epsilon_{F}(\hat{f}) = \int_{\mathcal{X}\times\{0,1\}\times\mathcal{Y}} l_{\hat{f}}(x,a,y)\; p_F(x,a,y) dx dady
\end{equation}

We also define the factual loss for the treatment ($a=1$) and control ($a=0$) groups respectively as:
\begin{equation}
\epsilon_{F}^{a=1}(\hat{f}) = \int_{\mathcal{X}\times\mathcal{Y}} l_{\hat{f}}(x,1,y)\; p_F(x,y|a=1) dx dy
\end{equation}
and 
\begin{equation}\epsilon_{F}^{a=0}(\hat{f}) = \int_{\mathcal{X}\times\mathcal{Y}} l_{\hat{f}}(x,0,y)\; p_F(x,y|a=0) dx dy
\end{equation}
\end{definition}

\begin{definition}[Counterfactual Loss]
The counterfactual loss is defined as:
\begin{equation}
\epsilon_{CF}(\hat{f}) = \int_{\mathcal{X} \times\{0,1\}  \times\mathcal{Y}} l_{\hat{f}}(x,a,y)\; p_{CF}(x,a,y) dx da dy
\end{equation}

We also define the counterfactual loss for the treatment ($a=1$) and control ($a=0$) groups respectively as:
\begin{equation}
\epsilon_{CF}^{a=1}(\hat{f}) = \int_{\mathcal{X}\times\mathcal{Y}} l_{\hat{f}}(x,1,y)\; p_{CF}(x,y|a=1) dx dy
\end{equation}
and 
\begin{equation}
\epsilon_{CF}^{a=0}(\hat{f}) = \int_{\mathcal{X}\times\mathcal{Y}} l_{\hat{f}}(x,0,y)\; p_{CF}(x,y|a=0) dx dy
\end{equation}
\end{definition}

The counterfactual loss corresponds to the expected loss value in a  parallel universe where the roles of the control and treatment groups are exchanged.
\begin{definition}
The \textit{Expected Precision in Estimating Heterogeneous Treatment Effect} (PEHE) is defined as: 
\begin{equation}
     \varepsilon_{P E H E}(\hat{f})=\int_{\mathcal{X}}\left(\hat{\tau}(x)-\tau(x)\right)^2 p_F(x) d x.
\end{equation}
\end{definition}
Here, $\varepsilon_{PEHE}$~\citep{hill} is often used as the performance metric for estimation of ITEs~\citep{shalit,johanson}. A critical connection between the factual loss ($\epsilon_F$), the counterfactual loss ($\epsilon_{CF}$), and $\varepsilon_{PEHE}$ is that for small values of $\epsilon_F$ and $\epsilon_{CF}$  causal models have good performance (i.e., low $\varepsilon_{PEHE}$). However, the $\varepsilon_{PEHE}$ is not directly accessible in causal inference scenarios because the calculation of $\tau(x)$ (i.e., the ground truth ITE values) requires access to the counterfactual values. In this light, we choose a hypothesis that instead optimizes an upper bound of $\varepsilon_{PEHE}$ given in Equation~\ref{opt_obj}.

\subsection{Representation Learning for ITE Estimation}
In this work, we consider The TARNet model~\cite{shalit} for causal learning. TARNet was developed as a framework to estimate ITEs using counterfactual balancing. It consists of a pair of functions $(\Phi,h)$ where $\Phi: \mathbb{R}^{d}\rightarrow \mathbb{R}^{l}$ is a representation learning function, and $h:\mathbb{R}^{l}\times \{0,1\}\rightarrow \mathbb{R}$ is a function learning the two potential outcomes functions in the representation space. The hypothesis learning for the true causal function is $\hat{f}(x, a) = h(\Phi(x), a)$ and the loss function $\ell_{\hat{f}}$ is denoted by $\ell_{(\Phi,h)}$. To ensure the similarity between the features of the treatment group and that of the control group in the representation space, TARNet uses the \textit{Integral Probability Metric} in order to measure the distance between distributions, defined as:
%TARNet uses integral probability metric (IPM) defined as 
\begin{equation}
\label{ipm}
\underset{G}{\text{IPM}}(p, q):=\sup _{g \in G}\left|\int_{S} g(s)(p(s)-q(s)) d s\right|
\end{equation}
where the supremum is taken over a given class of functions $G$.
It follows from the Kantorovich-Rubinstein duality \cite{villani} that $\text{IPM}$ reduces to the 1-Wassertein distance when $G$ is the set of 1-Lipschtiz functions as is the case in our numerical experiments. Here, the TARNet model learns to estimate the potential outcomes by minimizing the following objective: 
\begin{multline}
\label{opt_obj}
\mathcal{L}(\Phi,h) =  \frac{1}{N} \sum_{i=1}^N w_{i} \cdot \ell_{(\Phi,h)}(x_i,a_{i},y_{i})\\+\alpha \cdot \underset{G}{\text{IPM}}\left(\left\{\Phi\left(x_i\right)\right\}_{i: a_i=0},\left\{\Phi\left(x_i\right)\right\}_{i: a_i=1}\right)
\end{multline}

where $\displaystyle w_{i}=\frac{a_i}{2 v}+\frac{1-a_i}{2(1-v)}$, $\displaystyle v=\frac{1}{N} \sum_{i=1}^N a_i$, and $\alpha$ is the \textit{balancing weight} which controls the trade-off between the similarity of the representations in the latent domain and the model's performance on the factual data.

\section{Theoretical Framework}\label{sec:theory}
In this section, we provide learning bounds on the counterfactual loss of the target task, and $\varepsilon_{PEHE}$ (i.e., the error in estimating ITE). These bounds are inspired by the work of \cite{Ben-David2010} in the non-causal setting. We use superscripts $T$ and $S$ to respectively denote quantities related to the target and source tasks. Let $\tau^{T}$ denote the individual treatment effect function of the target task. We consider the performance of a well-trained source model $\hat{f}^{S}:\mathcal{X}\times \{0,1\} \to \mathcal{Y}$ when applied to a target task:
\begin{multline}
    \varepsilon^{T}_{PEHE}(\hat{f}^S) = \\
    \underset{x \sim p_{F}^{T}}{\mathbb{E}}\left[\left(\tau^{T}(x)-[\hat{f}^{S}(x,1) - \hat{f}^{S}(x,0)]\right)^2\right]
\end{multline}

% It is crucial to notice the uniqueness of transferring causal knowledge. Unlike regular transfer learning, where the objective is to minimize the loss over the factual distribution of the target, success in the target's factual distribution does not ensure success in causal inference, as shown by Theorem \ref{thm:lb_epehe}."

\subsection{The Challenge of ITE Knowledge Transfer}
We first provide a lower bound on $\varepsilon_{PEHE}$ that consists of both the factual and the counterfactual losses. This bound implies that good performance on the counterfactual data is a \textit{necessary} condition for accurate estimation of ITE. 

\begin{theorem}
\label{thm:lb_epehe}
Let $\hat{f}^{S}$ be a model trained on a source task, and $u = p^{T}_{F}(A=1) $ then
\begin{align}
    \epsilon^{T}_{F}(\hat{f}^{S}) + u \epsilon^{T,a=0}_{CF}(\hat{f}^{S}) \leq \varepsilon_{PEHE}^{T}(\hat{f}^{S})
\end{align}

\end{theorem}
According to the bound in Theorem~\ref{thm:lb_epehe}, simply minimizing the factual loss of the target may not guarantee a good performance. Hence, choosing a source model with low (or zero) factual loss on the target task cannot perform well if the (immeasurable) counterfactual loss of the target becomes excessively high. In other words, the performance of the chosen source model can be arbitrarily inadequate, while its performance appears perfect on factual data. 

While Theorem~\ref{thm:lb_epehe} has implied that causal knowledge cannot be transferred without any assumption, the learning bounds presented in the following section prove the viability of transferring causal knowledge under reasonable assumptions.

\subsection{General Learning Bounds}
The problem of ITE knowledge transfer can be expressed as two triples $(p^{S}_{F},p^{S}_{CF},f^{S})$ and $(p^{T}_{F},p^{T}_{CF},f^{T})$ where:
\begin{itemize}
    \item $p^{S}_{F}$ and $p^{T}_{F}$ respectively denote the factual probability distribution of the source and target tasks. 
    \item $p^{S}_{CF}$ and $p^{T}_{CF}$  respectively denote the counterfactual distribution of the source and target tasks.
    \item $f^{S}$ and $f^{T}$ respectively denote the underlying causal function of the source task and the target task.
\end{itemize}
We use the $L_1$ distance to measure the similarity between probability distributions, defined as: 
\begin{equation}
    V(p,q) = \int_{\mathcal{S}} |p(s) - q(s)| ds.
\end{equation}
 
 \begin{theorem}
\label{thm:smiple_d1}
For any hypothesis $\hat{f}$, we have:
\begin{align}
    \begin{aligned}
        \epsilon^{T}_{CF}(\hat{f}) \leq & \epsilon^{S}_{F}(\hat{f}) + 
        V(p^{T}_{F},p^{S}_{F}) + V(p^{T}_{F},p^{T}_{CF}) \\ & + \underset{{(x,a) \sim p^{S}_{F}}}{\mathbb{E}}[|f^{S}(x,a) - f^{T}(x,a)|]  
    \end{aligned}
\end{align}
and
\begin{align}
    \begin{aligned}
        \varepsilon^{T}_{PEHE}(\hat{f}) \leq & 4 \epsilon^{S}_{F}(\hat{f}) + 
        4 V( p^{T}_{F},p^{S}_{F}) + 2 V(  p^{T}_{F},p^{T}_{CF}) \\ & + 4 \underset{{(x,a) \sim p^{S}_{F}}}{\mathbb{E}}[|f^{S}(x,a) - f^{T}(x,a)|]  
    \end{aligned}
\end{align}

\end{theorem}
We note that the learning bounds consist of (1) the source factual loss, (2) the difference between the causal functions, and (3) a measure of similarities between probability distributions. However, the $L_1$ distance in Theorem~\ref{thm:smiple_d1} is intractable in practice. A more reasonable candidate distance is IPM distance as defined in Equation \ref{ipm}. The $L_1$ distance can be replaced with the IPM distance as demonstrated by the following Theorem~\ref{thm:simple_ipm}.
\begin{theorem}
\label{thm:simple_ipm}
Suppose that the function class $G$ is stable under addition and multiplication and $\hat f, f^{T} \in G$, then
\begin{align*}
    \begin{aligned}
        \epsilon^{T}_{CF}(\hat{f}) \leq & \epsilon^{S}_{F}(\hat{f}) + 
        \underset{G}{\text{IPM}}(p^{T}_{F},p^{S}_{F}) + \underset{G}{\text{IPM}}(p^{T}_{F},p^{T}_{CF}) \\ & + \underset{{(x,a) \sim p^{S}_{F}}}{\mathbb{E}}[|f^{S}(x,a) - f^{T}(x,a)|]  
    \end{aligned}
\end{align*}
and
\begin{align*}
    \begin{aligned}
        \varepsilon^{T}_{PEHE}(\hat{f}) \leq & 4 \epsilon^{S}_{F}(\hat{f}) + 
        4 \underset{G}{\text{IPM}}( p^{T}_{F},p^{S}_{F}) + 2 \underset{G}{\text{IPM}}(  p^{T}_{F},p^{T}_{CF}) \\ & + 4 \underset{{(x,a) \sim p^{S}_{F}}}{\mathbb{E}}[|f^{S}(x,a) - f^{T}(x,a)|]  
    \end{aligned}
\end{align*}
\end{theorem}

% We note that, due to inaccessibility to counterfactual data, $\underset{G}{\text{IPM}}(  p^{T}_{F},p^{T}_{CF})$, i.e., the distance between factual distributon and counterfactual distribution, is difficult to estimate in practice. 
% To overcome this problem, we next present generalization bounds for representation learning models that only involves factual distribution. 

\subsection{Bounds for Counterfactual Balancing Frameworks}
Suppose that we have a representation learning model (e.g., TARNet) $\hat{f}^{S} = (\Phi, h)$ trained on a source causal inference task. We apply the source model to a different target task. For notational simplicity, we denote $P(\Phi(X)|A=a)$ by $P(\Phi(X_{a}))$ for $a\in\{0,1\}$.
We make the following assumptions \textbf{A1, A2, A3}: 
\begin{itemize}
    \item \textbf{A1}: $\Phi$ is injective (thus $\Psi = \Phi^{-1}$ exists on $\text{Im}(\Phi)$).
    \item \textbf{A2}: There exists a real function space $G$ on $\text{Im}(\Phi)$ such that the function $r \mapsto \ell^{T}_{\Phi,h}(\Psi(r), a,y) \in G$.
    \item \textbf{A3}: There exists a function class $G'$ on $\mathcal{Y}$ such that $y\mapsto \ell_{\Phi,h}(x,a,y) \in G'$. 
    %and almost surely on $\mathcal{X}$ with respect to $P(X^{Sr})$.
\end{itemize}

% Note that the causal knowledge transferability assumption implies that the outcome distributions (causal effects) of treatment $t$ in source and target tasks need to be similar in order for transfer learning to be beneficial. 

The above theorem guarantees that causal knowledge can be transferred under reasonable assumptions. The following Lemma provides an upper bound on the counterfactual loss for transferring causal knowledge.
% the loss function is non-negative, i.e. $\ell^{Ta}_{h, \Phi}(x, t) \ge 0$ for all $x,t \in \mathcal{X}\times \{0,1\}$. And the representation function $\Phi$ is injective so that there is an inverse function $\Psi = \Phi^{-1}$ defined on the image of $\Phi$. There exist two constants: $B^{Ta}_{\Phi}$ and $B^{Sr}_{\Phi}$ such that for $t\in\{0,1\}$, the function $g^{Ta}_{\Phi, h}(r, t):=\frac{1}{B^{Ta}_{\Phi}} \cdot \ell^{Ta}_{h, \Phi}(\Psi(r), t) \in$ G and $g^{Sr}_{\Phi, h}(r, t):=\frac{1}{B^{Sr}_{\Phi}} \cdot \ell^{Sr}_{h, \Phi}(\Psi(r), t) \in$ G. Assume that for some function space $G'$ that includes polynomial functions up to highest degree of two, $IPM_{G'}(P(Y_t^{Sr}|x), P^(Y_t^{Ta}|x)) \le D$ for $t\in \{0,1\}$ and almost surely on $\mathcal{X}$ with respect to $P(X^{Sr}_{1}):= P(X^{Sr}|t=1)$.\\

% \begin{lemma} 
% \label{lemma1}
% Suppose that Assumptions 1-3 hold. The factual losses of any model $(\Phi,h)$ on source and target task satisfy for every $t \in \{0,1\}$:
% \begin{equation*}
%     \begin{aligned}
%         \epsilon_{F}^{T,t}(\Phi,h) \le & \epsilon_{F}^{T,t}(\Phi,h) + \delta \\
%         & + \underset{G}{\text{IPM}}(P(\Phi(X_t^{T})),P(\Phi(X_t^{S})))
%     \end{aligned}
% \end{equation*}
% \end{lemma}

\begin{lemma}
Suppose that Assumptions A1, A2, A3 hold. Then the counterfactual loss of any model $(\Phi,h)$ on the target task satisfies:
\label{lemma2}
\begin{equation*}
    \begin{aligned}
        \epsilon_{CF}^{T}(\Phi,h) \le &\epsilon_F^{S,a=1}(\Phi,h) + \epsilon_F^{S,a=0}(\Phi,h)\\ 
                               & + \underset{G}{\text{IPM}}(P(\Phi(X_1^{T})),P(\Phi(X_1^{S}))) \\
                               & + \underset{G}{\text{IPM}}(P(\Phi(X_0^{T})),P(\Phi(X_0^{S}))) \\
                               & + \underset{G}{\text{IPM}}(P(\Phi(X_0^{T})),P(\Phi(X_1^{T})))+2\gamma^*\\
    \end{aligned}
\end{equation*}
where 
\begin{equation}
%\label{transferability_assumption}
    \gamma^* = \underset{{x \sim p^S_F}}{\mathbb{E}}\left[\underset{G'}{\text{IPM}}(P(Y_a^{S}|x), P(Y_a^{T}|x))\right]
\end{equation}
measures the fundamental difference between two causal inference tasks.
\end{lemma}

\begin{theorem}{(Transferability of Causal Knowledge)} \label{main_theorem_1}Suppose that Assumptions A1, A2, A3 hold. The performance of source model on target task, i.e. $\varepsilon^{T}_{PEHE}(\Phi,h)$, is upper bounded by:
\begin{equation*}
\begin{aligned}
    \varepsilon^{T}_{PEHE}(\Phi, h) \le &2(\epsilon_F^{S,a=1}(\Phi,h) + \epsilon_F^{S,a=0}(\Phi,h)\\ 
    &+\underset{G}{\text{IPM}}(P(\Phi(X_1^{T})),P(\Phi(X_1^{S})))\\
    & + \underset{G}{\text{IPM}}(P(\Phi(X_0^{T})),P(\Phi(X_0^{S}))) \\
    & + \underset{G}{\text{IPM}}(P(\Phi(X_0^{T})),P(\Phi(X_1^{T}))+2\gamma^*) 
\end{aligned}
\end{equation*}
\end{theorem}

Theorem~\ref{main_theorem_1} implies that good performance on the target task is guaranteed if (1) the source model has a slight factual loss (e.g., the first and second term in the upper bound) and (2) the distributions of the control and the treatment group features are similar in the latent domain (e.g., the last three terms in the upper bound). This upper bound provides a sufficient condition for transfer learning in causal inference scenarios, indicating the transferability of causal knowledge.

\section{Task-Aware ITE Knowledge Transfer }\label{sec:task_distance}
In Section~\ref{sec:theory}, the regret bounds indicate the transferability of causal knowledge between pair of causal inference tasks. In this section, we propose a causal inference learning framework (illustrated in Figure~\ref{fig:intro_plot_1}) capable of identifying the most relevant causal knowledge, when multiple sources exist, to train the target task. 
Note that although the generalization bounds are informative for understanding viability of transferring causal knowledge, they may not be the most constructive approach to select the best source task because the order of the upper bounds of errors is not necessarily the same as the order of the errors.
To this end, we first propose a task affinity (CITA) that satisfies the symmetry property of causal inference tasks (see Sec~\ref{symmetry_of_causal_inference_task}) to find the closest source task to the target task. We observe that CITA strongly correlates with counterfactual loss. After obtaining the closest task using the computed task distances, its knowledge (e.g., trained model, bundled data) is utilized for training the target task.

\subsection{Task Affinity Score}
\label{sec:fisher}
%CITA is defined based on the Task Affinity Score (TAS). 
Let $(T, D)$ denote the pair of a causal inference task $T$ and its dataset $D = (X, A, Y)$, where $D$ consists of the covariates $X$, the corresponding treatment assignments $A$, and the factual outcomes $Y$. We formalize a \textit{sufficiently well-trained} deep network representing a causal task-dataset pair $(T, D)$ in Appendix (see Sec F). Here, all the previous tasks' models are assumed to be sufficiently well-trained to represent the corresponding tasks. 
% These models are denoted as $\varepsilon$-approximation networks. 
Next, we recall the definitions of the Fisher Information matrix and the Task Affinity Score~\citep{le2022task, 9766163}.

\begin{definition}[Fisher Information Matrix]
For a neural network $N_{\theta_{s}}$ with weights $\theta_{s}$ trained on data $D_{s}$, a given test dataset $D_t$ and the negative log-likelihood loss function $L(\theta,D)$, the Fisher Information matrix is defined as:
\begin{align}\label{Fishermatrix}
    F_{s,t} &=\mathbb{E}_{D\sim D_t}\Big[\nabla_{\theta} L(\theta_{s},D)\nabla_{\theta} L(\theta_{s},D)^T\Big] 
\end{align}
\end{definition}

\begin{definition}[Task Affinity Score]
\label{frechdist}
Let $(T_s, D_s)$ and $(T_t, D_t)$ denote the source and target task-dataset pairs, respectively. Let the source task be represented by the $\varepsilon$-approximation network $N_{\theta_s}$. Let $F_{s,s}$ be the Fisher Information matrix of $N_{\theta_s}$ using the source data $D_s$. Let $F_{s,t}$ be constructed analogously using the target data $D_t$ on $N_{\theta_s}$. The distance from the source task $T_s$ to the target task $T_t$ is defined as:
\begin{align}\label{frechet_org}
    d[s,t] &= \frac{1}{\sqrt{2}} \|F_{s,s}^{1/2} - F_{s,t}^{1/2}\|_F,
\end{align}
\end{definition}
where the norm is the Frobenius norm. It has been shown that $0 \le d[s,t] \le 1$, where $d[s,t]=0$ denotes perfect similarity and $d[s,t]=1$ indicates perfect dissimilarity. In Appendix (see Theorem F.3), we prove that under stringent assumptions, the order of task distances between candidate source tasks and the target task is preserved in a parallel universe where the roles of the control and treatment groups are exchanged. 

\label{method}
\subsection{Causal Inference Task Affinity (CITA)}
\label{symmetry_of_causal_inference_task}
\textbf{Symmetry of Causal Inference Tasks.} We first observe that causal inference tasks present a unique symmetry. Specifically, causal inference tasks have multiple regression problems, one for each treatment group. Given a source task, if we alternate the treatment labels (i.e., $0$ to $1$ and $1$ to $0$), the treatment effect (i.e., $\mathbb{E}[Y_{1}-Y_{0}|X]$) will be negated. Consequently, the non-symmetric task affinity measure ~\citep{le2022task} between the original task and the permuted task can still be considered. Moreover, the original model does not need to be retrained for transfer learning as we only need to permute the roles of output layers of the model to predict the individual treatment effects correctly for each group. In other words, the causal task affinity between these two permuted tasks must intuitively equal to zero. CITA lends itself to this property of causal inference tasks.

\begin{figure*}[t]
\centering
\centerline{\includegraphics[width=0.78\textwidth]{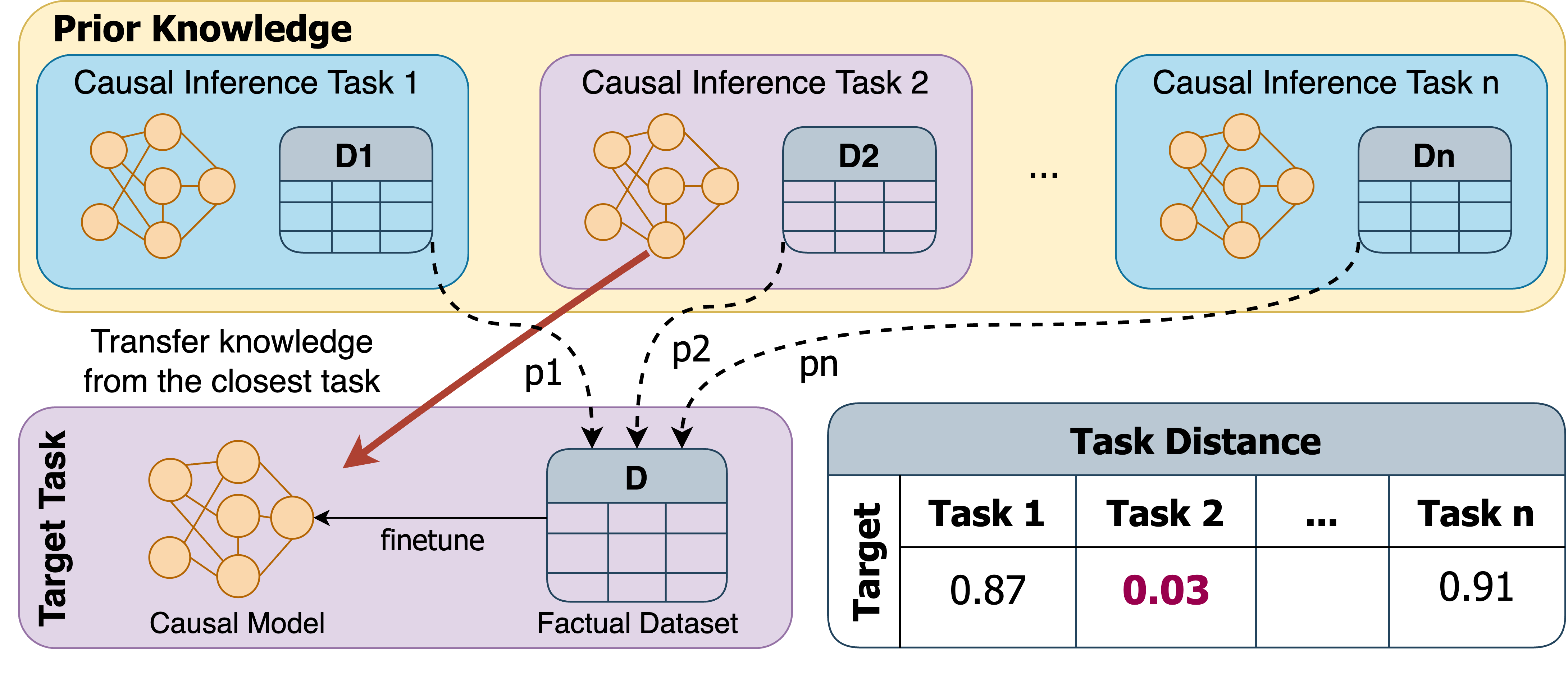}}
\caption{Overview of transfer learning in causal inference. Task affinity (CITA) is used to identify the closest task(s) from prior tasks. The models and datasets from the relevant prior tasks are transferred to the target task.}
\label{fig:intro_plot_1}
\end{figure*}

\begin{algorithm}[t]
\SetKwInput{KwInput}{Input}         % Set the Input
\SetKwInput{KwOutput}{Output}       % set the Output
\SetKwInput{KwData}{Data}
\SetKwFunction{main}{Main}
\SetKwFunction{dis}{TAS}
\SetKwComment{Comment}{$\triangleright$ }{}
\SetCommentSty{scriptsize}

\DontPrintSemicolon

\KwData{Source tasks: $\mathcal{S} = \{\{(X_i, A_i, Y_i)\}, 1\leq i \leq m \}$, Target task: $T = (X_t, A_t, Y_t)$}
\KwInput{Causal Inference Models $N_{\theta_1}, N_{\theta_2},\ldots, N_{\theta_m}$}
\KwOutput{Causal Inference model for the target task $T$}

    \SetKwProg{Fn}{Function}{:}{}
    \Fn{\dis{$X_s, A_s, X_s, A_s, N_{\theta_s}$}}{
        Compute $F_{s,s}$ using $N_{\theta_s}$ with $X_s, A_s$\;
        
        Compute $F_{s,t}$ using $N_{\theta_t}$ with $X_t, A_t$\;
        
        \KwRet $\displaystyle d[s, t] = \frac{1}{\sqrt{2}} \|F_{s,s}^{1/2} - F_{s,t}^{1/2}\|_F$
    }
    % Write task affinity Function
    \SetKwProg{Fn}{Function}{:}{} \Fn{\main}{
        \Comment*[r]{Find the closest tasks in S}
        \For{$i=1,2,\ldots,m$}{ 
            Train $N_{\theta_i}$ for source task $i$ using $(X_i, A_i,Y_{i})$\;
            
            Compute the distance from source task $i$ to target task $T$: \;\quad$d_i^+ = \dis(X_i,A_i, X_t,A_t, N_{\theta_i})$\;
            
            Compute the distance from source task $i$ to target task $T'$, where $a'$'s treatments are inverted treatments of $a$: \;\quad$d_i^- = \dis(X_i,A_i, X_t,1-A_t, N_{\theta_t})$\;
            
            CITA: $s_{sym_i} = \mathrm{min} (d_i^+, d_i^-)$
       }
       \KwRet closest tasks: $i^* = \underset{i}{\mathrm{argmin}}\ d_{sym_i}$\;
       
       \Comment*[r]{ITE Knowledge Transfer}
       Fine-tune $N_{\theta_i^*}$ with the target task's data $(X_t, A_t,Y_t)$\;
    }
    \KwRet $N_{\theta_i^*}$
\caption{Task-Aware ITE Knowledge Transfer}
\label{alg}
\end{algorithm}

\label{affinity}
\textbf{Properties of CITA.} In this work, assume that all causal inference tasks under consideration have the same number of treatment labels, and each task is represented by a TARNet network. Let $(T_{s}, D_{s})$ and $(T_{t}, D_{t})$ be the source and target causal inference tasks, respectively. Here, $D_{s} = (X_{s},A_{s},Y_{s})$, $D_{t} = (X_{t},A_{t},Y_{t})$, and $A_{s}, A_{t} \in \{0,1, \ldots,M\}$. 

Consider the symmetric group $\mathbb{S}_{M+1}$ consisting of all permutations of labels $\{0,1, \ldots,M\}$. For $\sigma \in \mathbb{S}_{M+1}$, let $A_{\sigma(t)}$ denote the permutation of the target treatment labels under the action of $\sigma$. Let $\displaystyle d_{\sigma} = \frac{1}{\sqrt{2}} \|F_{a,a}^{1/2} - F_{a,\sigma(t)}^{1/2}\|_F$
then $$d_{sym}[s, t] = \min_{\sigma \in \mathbb{S}_{M+1}} (d_{\sigma})$$ is the \textit{label-invariant task affinity} between causal tasks $T_s$ and $T_t$. Similar to the Task Affinity Score~\citep{le2022task}, the proposed distance is robust to the architectural choice of the representation networks (e.g., TARNet). In other words, the order of closeness of tasks found using this distance remains the same across different choices of network architecture or hyper-parameters. Notably, our approach involves a minimization problem with $(M+1)!$ candidates.  However, in real-life applications, the number of treatments $M+1$ is often a a small number. It is nevertheless still important to note that it may be challenging to apply our method in the scenarios with large $M$ or for a continuum of treatments.

\subsection{Causal Knowledge Transfer from CITA}
Based on CITA, we propose a task-aware causal inference learning framework, whose procedure is illustrated in Figure~\ref{fig:intro_plot_1}, that is capable of utilizing past experiences to learn the target task's ITE quickly. In particular, given multiple trained source tasks, the closest task is identified via causal inference task affinity (CITA). Subsequently, its knowledge (e.g., trained causal model, weights, parameters, initialization settings) is applied to learn the causal effect of the target task. Here, the source task's model is fine-tuned with the target task's data for estimating ITEs. In our experiments, we compare the performance of our method to training from scratch. We also compare our method to data-bundling, as illustrated in Figure 1 in the Appendix. The pseudo-code of the proposed framework is provided in Algorithm~\ref{alg}.

\begin{table}[t]
\caption{Overview of the causal inference datasets constructed for Transfer Learning.}
\label{dataset_table}
\begin{center}
\begin{small}
\begin{sc}
\begin{tabular}{l|ccc}
\hline
\multicolumn{1}{l}{\bf Name}  &\multicolumn{1}{c}{\bf Type} &\multicolumn{1}{c}{\bf Task}&\multicolumn{1}{c}{\bf CF Avail}\\ 
\hline
IHDP&Semi-Synthetic &REG&YES \\
Twins&Real-world &CLS &NO \\
Jobs &Real-world &CLS &NO \\
RKHS& Synthetic &REG &YES \\
Movement &Synthetic &REG &YES\\
Heat &Synthetic &CLS &YES \\
\hline
\hline
\multicolumn{1}{l}{\bf Name} &\multicolumn{1}{c}{\bf Subject} &\multicolumn{1}{c}{\bf \#Task} &\multicolumn{1}{c}{\bf \#Sample}\\
\hline
IHDP&Health &100 &747\\
Twins&Health&11&2000\\
Jobs&Social Sciences &10&619  \\
RKHS& Mathematics &100&2000\\
Movement &Physics&12&4000\\
Heat&Physics&20&4000\\
\hline
\end{tabular}
\end{sc}
\end{small}
\end{center}
\footnotesize{\textbf{REG/CLS}: Regression/Classification, \textbf{CF Avail}: Counterfactual Data Availability}
\end{table}

\begin{table}
\caption{The impact of causal knowledge transfer on the performance and the required size of the training dataset.}
\label{performance_table}
\begin{center}
\begin{small}
\begin{sc}
\begin{tabular}{lcccc}
\toprule
\multicolumn{1}{l}{\bf Dataset} & IHDP & RKHS & Movement & Heat \\ 
\midrule
\multicolumn{1}{l}{\bf ORI Size} & 747 &2000 &4000 &4000 \\
\multicolumn{1}{l}{\bf TL Size} & 150 & 50 & 750 & 500 \\
\multicolumn{1}{l}{\bf W/O TL(I)} &0.61 &0.68 &0.021&6.7e-6 \\
\multicolumn{1}{l}{\bf W/O TL (P)} &0.97 &0.96 &0.025 &1.4e-5 \\
\multicolumn{1}{l}{\bf W TL (P)} &0.65 &0.46 & 0.011 &4.2e-6 \\
\multicolumn{1}{l}{\bf Data Gain} & $>80\%$ &  $>95\%$ & $>80\%$ & $>85\%$\\ 
\multicolumn{1}{l}{\bf Perf Gain} & $>30\%$ &  $>50\%$ & $>55\%$ & $>70\%$ \\
\bottomrule
\end{tabular}
\end{sc}
\end{small}
\end{center}
\footnotesize{\textbf{ORI/TL Size}: Number of data required without and with TL, \textbf{W/O TL (I \& P)}: Performance achieved without TL (\emph{ideal} \& \emph{practice}); (ideal) is the model with the lowest $\varepsilon_{PEHE}$ (not attainable); (practice) is the model with the lowest training loss,  \textbf{W TL (P)}:  Performance with TL, \textbf{Data Gain}: Data Reduction with TL, \textbf{Perf Gain}: Error Reduction with TL.}
\end{table}

\section{Experiments}
\label{result}
We first describe the datasets we have used for our empirical studies. Subsequently, we present empirical results that (1) show that CITA identifies the symmetries within causal inference tasks, (2) demonstrate the strong correlation between CITA and the counterfactual loss, and (3) quantify the gains of transfer learning for causal inference.

\subsection{Causal Inference Datasets}
We present a representative family of causal inference datasets suitable for studying ITE knowledge transfer. Some of these are well-established datasets in the literature, while others are motivated by known causal structures in diverse areas such as social sciences, physics, health, and mathematics.
Table~\ref{dataset_table} provides a brief description of the datasets used in our studies. A more detailed description is provided in Appendix (see Sec B). For each dataset, a number of corresponding causal inference tasks exist, which can be used to study transfer learning scenarios. Please note that we can only access the counterfactual data of the synthetic/semi-synthetic datasets (i.e., IHDP, RKHS, Movement, Heat). We do not possess the counterfactual data of real-world datasets (i.e., Twins, Jobs).

\subsection{The symmetry of CITA}
Our numerical results for the Jobs and Twins datasets verify that the CITA can capture the symmetries within causal inference problems. We flip treatment labels  ($0$ and $1$) with probability $p$ (without any changes to the features and the outcomes) independently for each control and treatment data point. In Figure~\ref{fig:td_sym_plot}, we depict the trend of CITA between the original and the altered dataset by varying $p$, $p \in [0,1]$. The symmetry of CITA is evident (with some deviation due to limited training data for calculating CITA). The altered dataset with $p=1$ is the closest to the original dataset (as it should be) since we have completely flipped the treatment assignments. The altered dataset with $p=0.5$ is the furthest (as it should be) since we have randomly shuffled the control and the treatment groups. We also compare CITA with the nonsymmetrized task affinity~\citep{le2022task} on the Jobs and the Twins datasets. Figure~\ref{fig:td_sym_plot} shows that CITA has successfully captured the symmetries within causal inference tasks. We observe that CITA demonstrates symmetry at $p=0.5$, indicating the symmetry of the causal inference tasks. In contrast, the original nonsymmetrized task affinity fails to capture this symmetry property.

\begin{figure}[t]
\begin{center}
    \centering
    \includegraphics[width=0.23\textwidth]{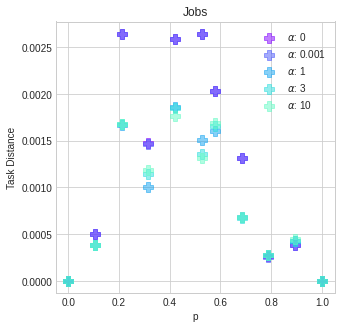}
    \centering
    \includegraphics[width=0.23\textwidth]{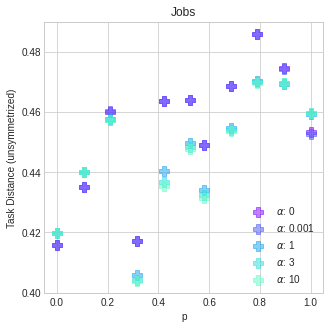}
    \centering
    \includegraphics[width=0.23\textwidth]{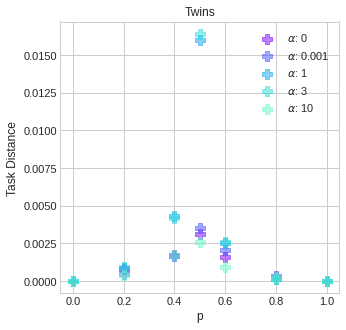}
    \centering
    \includegraphics[width=0.23\textwidth]{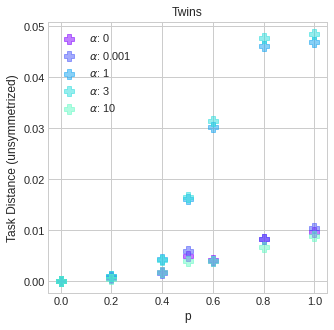}
    \caption{The symmetry of CITA. $p$ (on the x-axis) denotes the probability of flipping treatment assignments of the original dataset. Left column: CITA; right column: nonsymmetrized task affinity.}%Thus, the trend of points should resemble an inverted `U', successfully captured by the task affinity.}
    \label{fig:td_sym_plot}
\end{center}
\end{figure}

%Table for TL Performance
% \A{Table needs to be fixed to fit into page}
% \begin{table}[t]
% \caption{The impact of causal knowledge transfer on the size of the required training dataset. Number of training data used without and with TL (\textbf{ORI/TL Size}); Minimum $\varepsilon_{PEHE}$ achieved during training (not obtainable in practice because no validation data is available) (\textbf{W/O TL (Ideal)}); $\varepsilon_{PEHE}$ of the \textbf{model without transfer learning} 
% (model with minimum training loss)(\textbf{W/O TL (Prac)}); $\varepsilon_{PEHE}$ of the model with transfer learning (model with minimum training loss)(\textbf{TL (Prac)}); Percentage of Reduction in Data provided by Causal Transfer Learning (\textbf{Gain}).}
% \label{performance_table}
% \begin{center}
% \begin{small}
% \begin{sc}
% \begin{tabular}{l|ccccc}
% \hline
% \multicolumn{1}{l}{\bf Dataset} &\multicolumn{1}{c}{\bf ORI/TL Size} &\multicolumn{1}{c}{\bf W/O TL(Ideal)} &\multicolumn{1}{c}{\bf W/O TL (Prac)} &\multicolumn{1}{c}{\bf TL (Prac)} &\multicolumn{1}{c}{\bf Gain}\\ 
% \hline
% IHDP&747/150&0.61&0.97&0.65& $>80\%$ \\
% RKHS&2000/50&0.68&0.96&0.46&  $>95\%$ \\
% Movement&4000/750  &0.021&0.025& 0.011 & $>80\%$ \\
% Heat&4000/500 &6.7e-6 &1.4e-5 &4.2e-6 & $>85\%$ \\
% \hline
% \end{tabular}
% \end{sc}
% \end{small}
% \end{center}
% \end{table}

\begin{figure}[t]
\begin{center}
         \hfill
         \centering
         \includegraphics[width=0.23\textwidth]{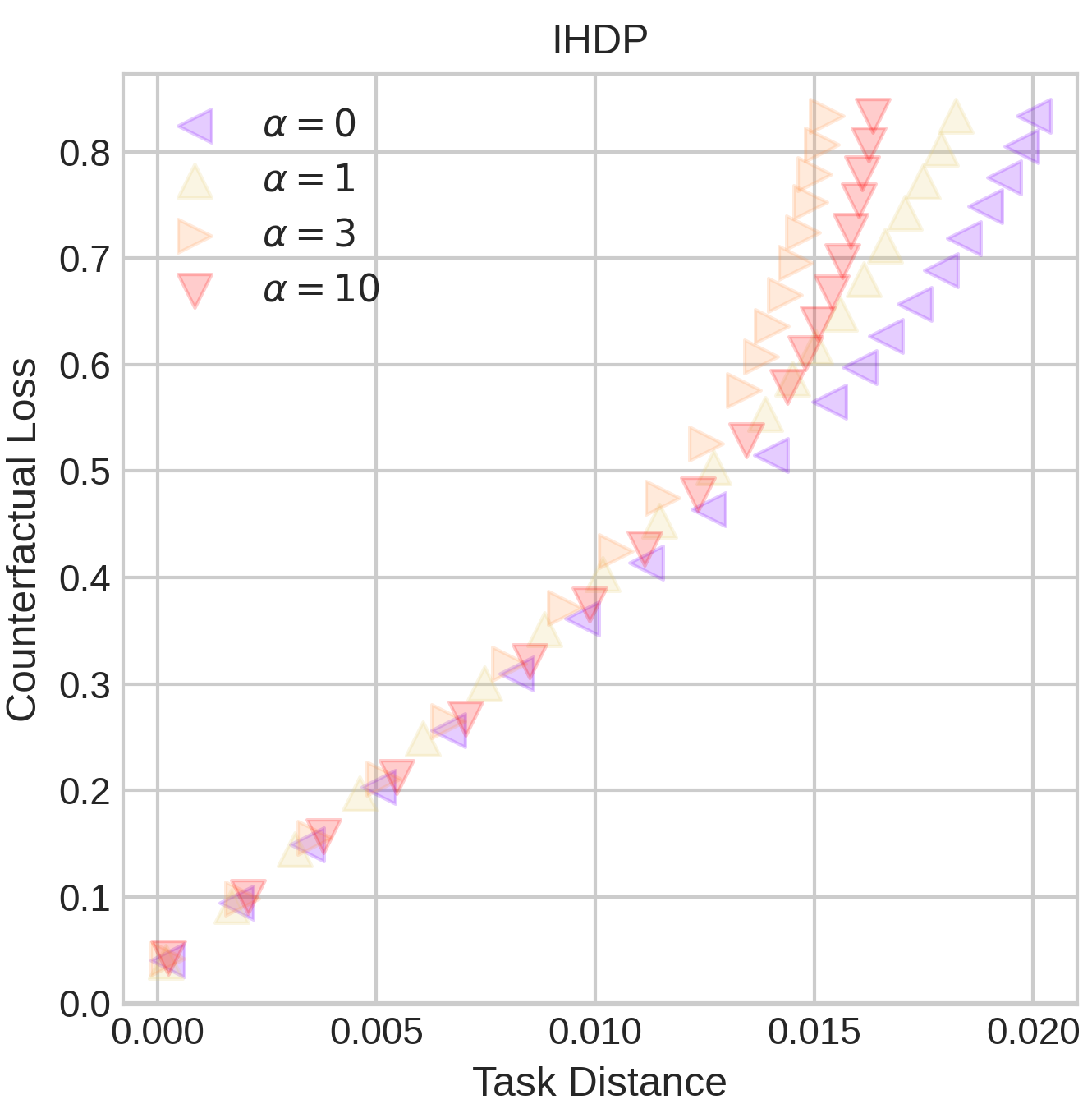}
         % \caption{IHDP}
         \label{fig:three sin x}
     \hfill
         \centering
         \includegraphics[width=0.23\textwidth]{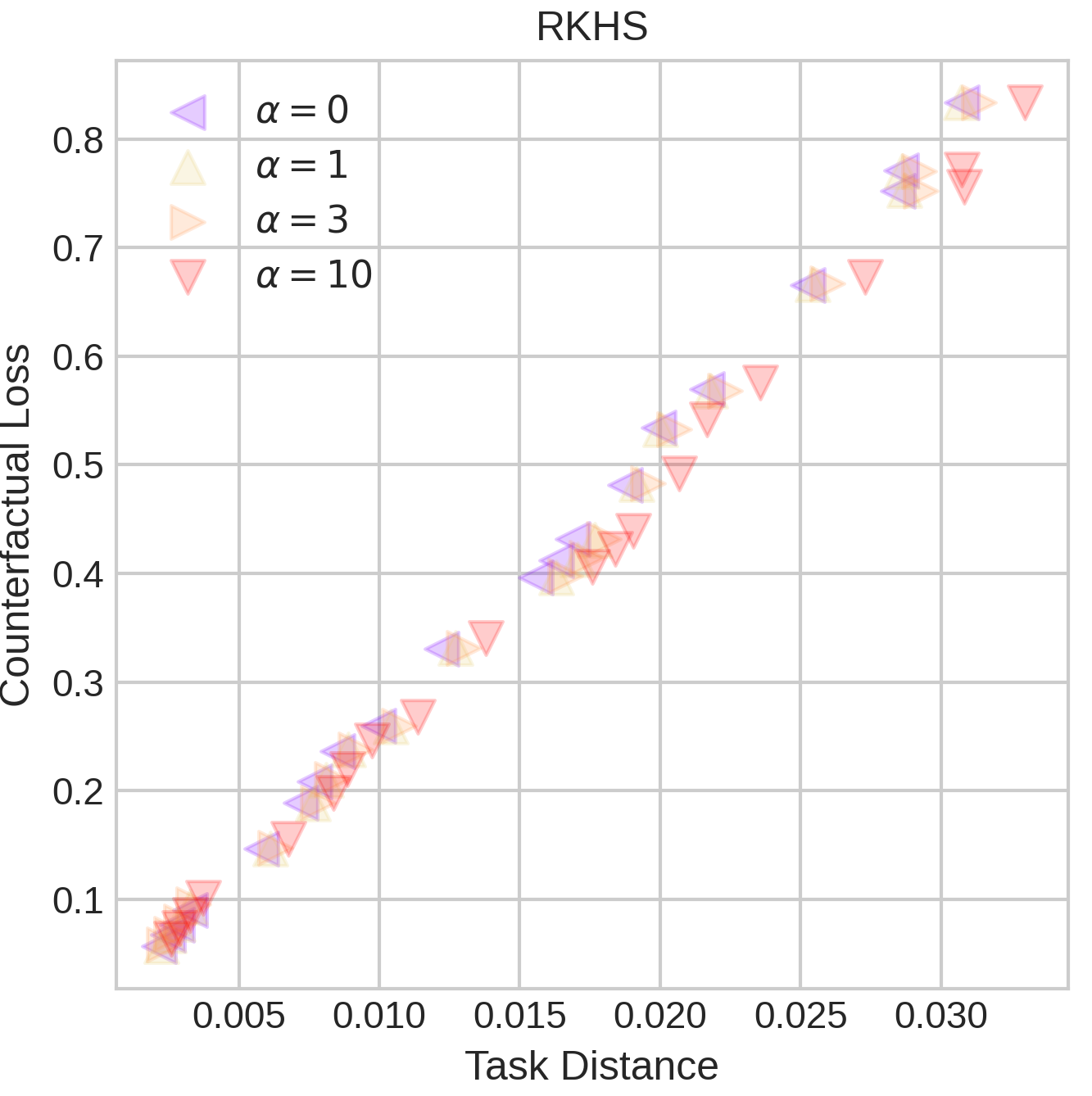}
         % \caption{Non-Linear}
         \label{fig:five over x}
     %\hspace{-0.2in}
     \hfill
         \centering
         \includegraphics[width=0.23\textwidth]{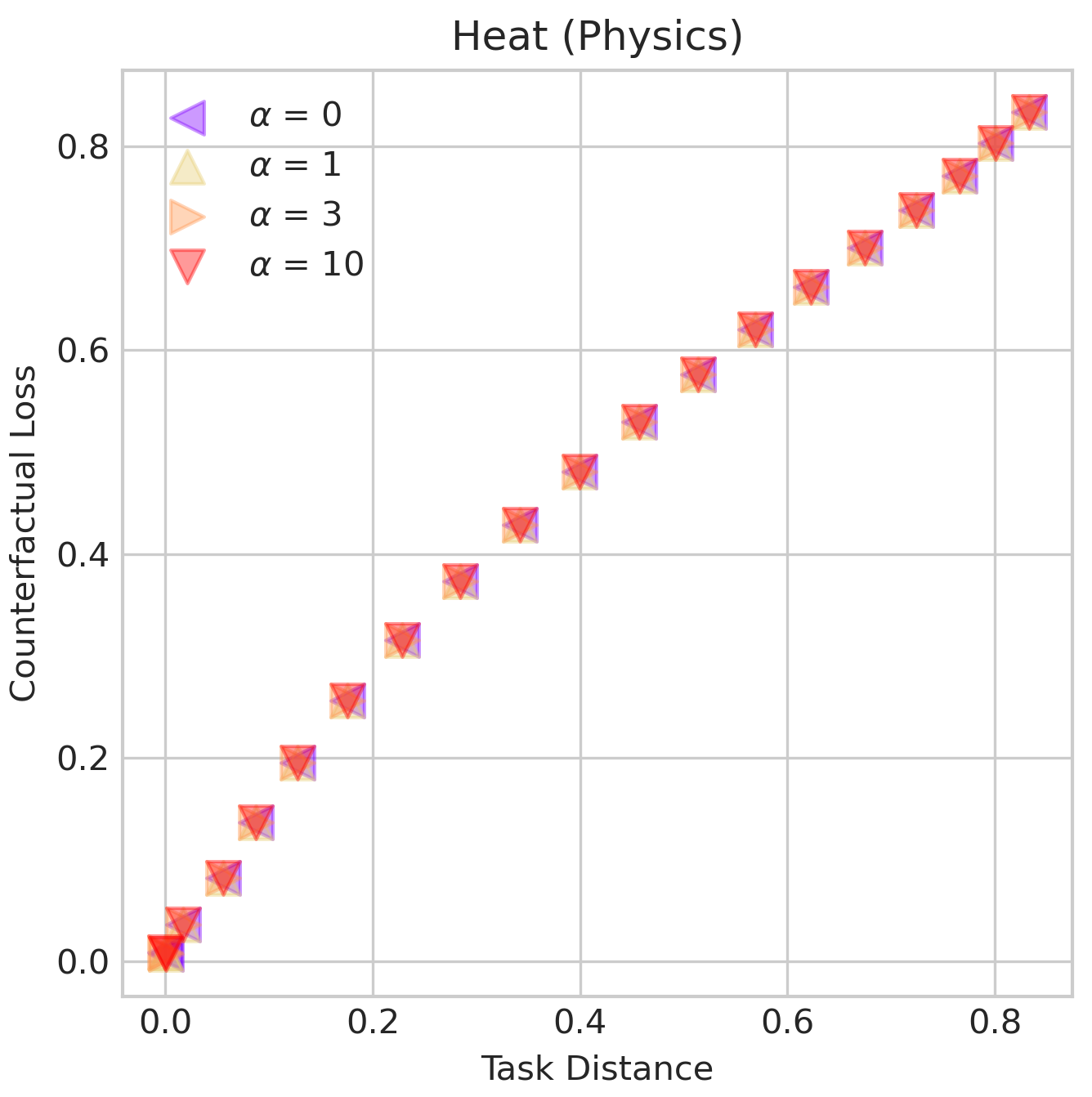}
         % \caption{Physics Heat Dataset}
         \label{fig:y equals x}
     \hfill
         \centering
         \includegraphics[width=0.23\textwidth]{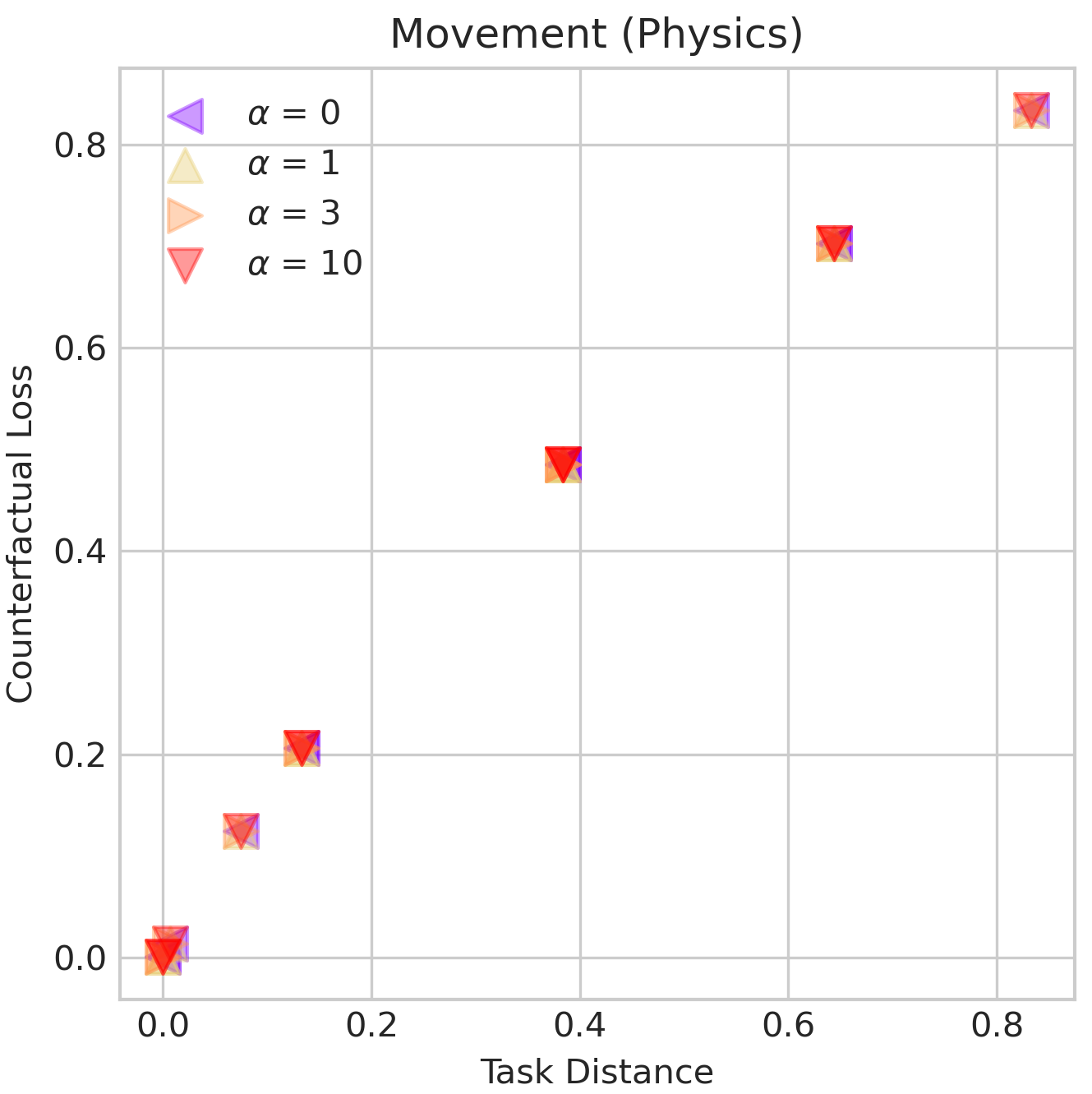}
         % \caption{Physics Movement Dataset}
         \label{fig:three sin x 2}
      \caption{CITA vs. Counterfactual Error on causal inference datasets. CITA strongly correlates with the (immeasurable) counterfactual loss.}
      \label{fig:td_cfe_plot}
\end{center}
\end{figure}

\subsection{CITA and The Counterfactual Loss}
\label{exp1}
 In this experiment, we empirically show the strong correlation between CITA (which only uses available data) and counterfactual loss (which is impossible to measure directly except for synthetic datasets). In Figure~\ref{fig:td_cfe_plot}, for different balancing weights $\alpha$ (see Equation~\ref{opt_obj}), we give the correlation between CITA and counterfactual error on the IHDP, RKHS, Movement(Physics), and Heat(Physics) datasets (for which counterfactuals are known). It is both intuitively appealing and empirically observed that CITA and the counterfactual loss have a strong correlation: the model of a source task has a smaller counterfactual loss on the target data if the target task is closer (in terms of CITA). 
Note that the points in Figure~\ref{fig:td_cfe_plot} for different values of $\alpha$ (i.e., balancing weight) are extremely close. In other words, CITA highly relates to the counterfactual loss and is robust to hyper-parameters shift. This property is desirable, especially in causal inference scenarios where no validation data can be accessed to cross-validate the hyper-parameters. Additionally, it can also be observed that CITA trends are robust to variations in the balancing weight for all datasets. 
% Plots of task affinity and CF error

\subsection{Comparison of Performance with/without Transfer Learning}\label{exp2}
This experiment aims to analyze the impact of transferring causal knowledge on the size of required training data. Here, we use Heat (Physics),  Movement (Physics), IHDP, and RKHS datasets for which the counterfactual outcomes are available. We first fix a target causal inference task. For a wide range of balancing weights ($\alpha$),  we record the values of $\varepsilon _{PEHE}$ for training the model from scratch while increasing the size of training datasets. In this process, the training datasets are slowly expanded such that smaller training sets are subsets of larger ones. We then report the minimum $\varepsilon_{PEHE}$ achieved for each dataset size. We identify the closest source task to the target task and repeat the above process with a small amount of target task data. We then compare the performance with and without transfer learning to quantify the amount of data needed by transfer learning models to achieve the best possible performance without transferring causal knowledge. The results are summarized in Table~\ref{performance_table}, which demonstrates that transferring causal knowledge decreases the required training data in this setting by $75\%$ to $95\%$. 

\section{Conclusion}
In this paper, we provided theoretical analysis proving the transferability of causal knowledge and outlined the underlying challenges. We also proposed a method for ITE transfer learning. Specifically, we constructed CITA, a new task affinity tailored for causal inference tasks, to measure the similarity of causal inference tasks that captures the symmetry within them. Given a new causal inference task, we transferred the ITE knowledge from the closest task between all the available previously trained tasks. Simulations on a representative family of datasets provide empirical evidence for the gains of our method and the efficacy of CITA. 

\subsubsection*{Acknowledgments}
This work was supported in part by the National Science Foundation (NSF) under the National AI Institute for Edge Computing Leveraging Next Generation Wireless Networks Grant \#  2112562.

\bibliography{aloui_189}

\onecolumn
\textbf{\huge{Appendix}}
%\section{Reproducibility Statement}
%The supplementary material includes the implementation codes for our proposed framework, TARNet, and CITA.
\section{Causal Inference: An Example}
Let $X \in \mathcal{X}$ be the features (e.g., age, height, weight), the treatment assignment $A \in \{0, 1\}$ be the indicator representing if the subject received vaccine $0$ or $1$. The mortality outcome is denoted by $Y\in\mathcal{Y}$. 
% (Without loss of generality, we can create one-to-one mapping between any finite set of values of $T$ and the set of integers ${0,...,M-1}$ where $M$ is the total number of values of $T$.)

The main challenge of causal inference arises from the absence of counterfactual observations. We do not observe the outcomes of individuals upon receiving treatment $1$ if they have received treatment $0$ and vice versa. The subjects who received vaccine $1$ may differ significantly from those who received treatment $0$. This issue is called selection bias. For instance, older people are more likely to receive the treatment than young people). Thus, estimating the counterfactual effects is challenging due to the unbalance between the treatment groups.

Let $\hat{f}(x,a)$ be a hypothesis modeling the outcome for an individual $x$ if he/she received treatment $a$. The factual loss is defined as follows:
\begin{equation}\epsilon_{F}(\hat{f}) = \int_{\mathcal{X}\times\{C,B\}\times\mathcal{Y}} l_{\hat{f}}(x,a,y)\; p(x,a,y) dx dady\end{equation}
By Bayes rule, we can write the factual loss as
\begin{equation*}
\begin{aligned}
&\epsilon_{F}(\hat{f}) \\
&= \int_{\mathcal{X}\times\mathcal{Y}} l_{\hat{f}}(x,a=0,y)\; p(x,y|A=0) p(A=0) dxdy +\\
&\int_{\mathcal{X}\times\mathcal{Y}}l_{\hat{f}}(x,a=1,y)\; p(x,y|A=1) p(A=1) dxdy \\
&= p(A=0) \int_{\mathcal{X}\times\mathcal{Y}} l_{\hat{f}}(x,a=0,y)\; p(x,y|A=0) dxdy +\\
&\left(1-p(A=0)\right)\int_{\mathcal{X}\times\mathcal{Y}}l_{\hat{f}}(x,a=1,y)\; p(x,y|A=1) dxdy \\
& = p(A=0) \epsilon_{F}^{A=0}(\hat{f}) + \left(1-p(A=0)\right) \epsilon_{F}^{A=0}(\hat{f})
\end{aligned}
\end{equation*}

We define the factual loss for the group who received vaccine $0$ as follows:

\begin{equation}
\epsilon_{F}^{A=0}(\hat{f}) = \int_{\mathcal{X}\times\mathcal{Y}} l_{\hat{f}}(x,a=0,y)\; p(x,y|A=0) dxdy
\end{equation}

Similarly, the factual loss for the group who received vaccine $1$ is described as: 

\begin{equation}
\epsilon_{F}^{A=1}(\hat{f}) = \int_{\mathcal{X}\times\mathcal{Y}} l_{\hat{f}}(x,a=1,y)\; p(x,y|A=1) dxdy
\end{equation}

Consider a parallel universe where the treatment assignments are flipped (i.e., those who received vaccine $1$ receive vaccine $0$ and vice versa). The performance of our hypothesis $\hat{f}$ in this scenario is the counterfactual loss, defined as follows:

\begin{equation}
\epsilon_{CF}(\hat{f}) = \int_{\mathcal{X} \times\{0,1\}  \times\mathcal{Y}} l_{\hat{f}}(x,a,y)\; p(x,1-a,y) dx da dy
\end{equation}

%-----------------------------------------------------------------------------

\section{Datasets and Experiments Descriptions}
\subsection{Datasets}
\label{datasets}
\paragraph{IHDP} 
The IHDP dataset was first introduced by \cite{hill} based on real covariates available from the Infant Health and Development Program (IHDP), studying the effect of development programs on children. The features in this dataset come from a Randomized Control Trial. The potential outcomes were simulated using Setting B. The dataset consists of $747$ individuals (e.g., $139$ in the treatment group and $608$ in the control group), each with $25$ features. The potential outcomes are generated as follows:
$$Y_{0} \sim \mathcal{N}(\exp(\beta^{T}\cdot(X + W)), 1)$$
and
$$Y_{1} \sim \mathcal{N}(\beta^{T}(X+W) - \omega,1)$$
where $W$ has the same dimension as $X$ with all entries equal $0.5$ and $\omega=4$. The regression coefficient $\beta$, a vector of length $25$, is randomly sampled from a categorical distribution with the support $(0, 0.1, 0.2, 0.3, 0.4)$ and the respective probabilities $\mu = (0.6, 0.1, 0.1, 0.1,0.1)$. The dataset generated according to these parameters is referred to as the \textit{base} dataset.

Additionally, we generate $9$ additional datasets by introducing $9$ new settings. These settings, which are constructed by varying $\mu$ and $\omega$, are shown in
Table~\ref{ihdp_table}. Each of these generated datasets consists of $747$ individuals (e.g., $139$ in the treatment group and $608$ in the control group).

%These values of $\mu$ and $\omega$ for the different IHDP datasets are given below in the Table.

\begin{table}[t]
\begin{center}
\caption{The settings to generate IHDP datasets}
\label{ihdp_table}
\begin{tabular}{l|lc}
\hline
\multicolumn{1}{l}{\bf Dataset} &\multicolumn{1}{c}{\bf $\mu$} &\multicolumn{1}{c}{\bf $\omega$} \\
\hline
IHDP (\textit{Base})&(0.6, 0.1, 0.1, 0.1, 0.1)&4\\
IHDP 1&(0.61, 0.09, 0.1, 0.1, 0.1) &4.1 \\
IHDP 2&(0.62, 0.08, 0.1, 0.1, 0.1) &4.2 \\
IHDP 3&(0.63, 0.07, 0.1, 0.1, 0.1) &4.3 \\
IHDP 4&(0.64, 0.06, 0.1, 0.1, 0.1) &4.4 \\
IHDP 5&(0.65, 0.05, 0.1, 0.1, 0.1) &4.5 \\
IHDP 6&(0.66, 0.04, 0.1, 0.1, 0.1) &4.6 \\
IHDP 7&(0.67, 0.03, 0.1, 0.1, 0.1) &4.7 \\
IHDP 8&(0.68, 0.02, 0.1, 0.1, 0.1) &4.8 \\
IHDP 9&(0.69, 0.01, 0.1, 0.1, 0.1) &4.9 \\
\hline
\end{tabular}
\end{center}
\end{table}

\paragraph{Jobs} The Jobs dataset~\citep{JobsDataset} consists of $619$ observations. In this experiment, the causal inference task aims to learn the effect of participation in a specific professional training program on landing a job in the following three years. Here, we generate a family of related datasets by randomly reverting the original treatment assignments (i.e., $0 \leftrightarrow 1$) with the probability $p \in \{0 = 0/9, 1/9, 2/9, 3/9, 4/9, 5/9, \cdots, 9/9=1\}$. The dataset corresponding to $p=0$ is considered the original dataset, and the dataset with $p=1$ has all treatment assignments reversed. We select the original Jobs dataset, introduced in~\citep{JobsDataset} as the \textit{base} dataset for our experiments.

\paragraph{Twins}
The Twins dataset~\cite{cevae} is based on the collected birthday data of twins born in the United States from 1989 to 1991. It is assumed that twins share significant parts of their features. Consider the scenario where one of the twins was born heavier than the other as the treatment assignment. The outcome is whether the baby died in infancy (i.e., mortality). Here, the twins are divided into two groups: the treatment and the control groups. The treatment group consists of heavier babies from the twins. On the other hand, the control group consists of lighter babies from the twins. All given observations from this dataset are considered factual.

We first construct a \textit{base} dataset by selecting a set of $2000$ pairs of twins from the original dataset~\citep{cevae}. Each individual is assigned to the treatment group according to a Bernoulli experiment with the probability of $q = 0.75$. In an analogous manner to that of the Jobs dataset, we generate a family of related datasets by randomly reverting the treatment assignments of the \textit{base} dataset (i.e., $0 \leftrightarrow 1$) with corresponding probabilities $p \in \{0, 0.1, 0.2, 0.3, 0.4, 0.5, \cdots, 1\} $. For instance, to generate dataset $i=1,2, \cdots, 11$, we revert the individual treatment assignments in the base dataset using the Bernoulli experiment with the probability of $p_i = (i-1)/10$. In particular, $p=0$ corresponds to the original dataset, while $p=1$ corresponds to all treatment assignments reverted.

% We choose $11$ different $p$ values ranging from 0 to 1 such that the $p$ values are evenly spaced (e.g., \texttt{numpy.linspace(0,1,11)}
% We choose the \textit{base} dataset as the \textit{base} dataset for our experiments~\ref{exp_detail}.  

% We generate multiple Twins datasets by flipping the treatment labels with probability $=p$ ranging from $0$ to $1$. This assures that we are changing both the selection bias and the potential outcomes functions. 

\paragraph{RKHS}
In this experiment, we generate $100$ Reproducing Kernel Hilbert Space (RKHS) datasets, each having $2000$ data points. Next, we generate the treatment and the control populations $X_{1},X_{0} \in \mathbb{R}^{4}$ respectively from Gaussian distributions $\mathcal{N}(\mu_{1},I_{4})$ and $\mathcal{N}(\mu_{0},I_{4})$ for each dataset. We sample $\mu_{1} \in \mathbb{R}^{4}$ and $\mu_{0}\in \mathbb{R}^{4}$ respectively according to Gaussian distributions $\mathcal{N}(\pmb{e},I_{4})$ and $\mathcal{N}(-\pmb{e},I_{4})$ where $\pmb{e}=[1,1,1,1]^{T}$.

Subsequently, we generate the potential outcome functions $f_{0}$ and $f_{1}$ with a Radial Basis Function (RBF) kernel $K(\cdot,\cdot)$, described as follows:

%Let $\gamma_{0}\in \mathbb{R}^{4}$ be a vector with its components sampled individually from the Gaussian distribution $\mathcal{N}(7,1)$. Let $\gamma_{1} \in \mathbb{R}^{4}$ be another vector with its components sampled individually from the Gaussian distribution $\mathcal{N}(9,1)$. \\
Let $\gamma_{0},\gamma_{1}\in \mathbb{R}^{4}$ be two vectors sampled  from $\mathcal{N}(7\pmb{e},I_{4})$ and $\mathcal{N}(9\pmb{e},I_{4})$, respectively. Let $\lambda \in \mathbb{N}$ be sampled uniformly from $\{10,11,\ldots,99,100\}$. For $j \in \{0,1\}$:
\begin{enumerate}
    \item We sample $m_{j} \in \mathbb{N}$ according to the Poisson distribution with parameter $\lambda$ (i.e., $\text{Pois}$)
    \item For every $i\in\{1,\ldots,m_{j}\}$, we sample $x_{j}^i$ according to  $\mathcal{N}(\gamma_{j},I_{4}) $
    \item The potential outcome functions $f_j, j=0,1$ are constructed as $f_{j}(\cdot) = \sum_{i=1}^{m_j}K(x_{j}^i,\cdot)$
\end{enumerate}
Given the potential outcome functions $f_j, j \in \{0,1\}$, the corresponding potential outcomes $Y_0$ and $Y_1$ are generated by:
$$
Y_0(x)=f_0(x), \; \text{for every}\; x\in \mathbb{R}^{4},
$$
and
$$
Y_1(x)=f_1(x), \; \text{for every}\; x\in \mathbb{R}^{4}.
$$
% Clearly, all the generated potential outcomes functions are in the same RKHS. 
 %We choose the \textit{base} dataset as the \textit{base} dataset for our experiments~\ref{exp_detail}.
We will refer to the first constructed dataset above as the \textit{base} dataset. Here, all the generated potential outcome functions are in the same RKHS.

\paragraph{Heat (Physics)} Consider a hot object left to cool off over time in a room with temperature $T(0)$. A person will likely suffer a burn if he/she touches the object at time $u$.

The causal inference task of interest is the effect of room temperature $T(0)$ on the probability of suffering a burn. This family consists of $20$ datasets; each includes $4000$ observations (e.g., $2000$ in the control group and $2000$ in the treatment group). The treatment in our setting is $a=1$ when $T(0) = 5$, and $a=0$ when $T(0) = 25$. The touching times of the treatment and control groups are sampled from two Chi-squared distributions $\chi^{2}(5)$ and $\chi^{2}(2)$, respectively, to introduce artificial bias. 

From the solution to Newton's Heat Equation~\citep{heat_equation}, the underlying causal structure is governed by the following equation:
$$T(u) = C \cdot\exp(-ku) + T(0)$$
where $T(u)$ is the temperature at time $u$ and $C, k$ are constants. Let $T_{0} = 25, C = 75$ for the control groups and $T_{0} = 5, C = 95$ for the treatment groups in the datasets. We choose $20$ values of $k= \{0.5, \cdots, 2\}$ uniformly spaced in $[0.5, 2]$. For each value of $k$, we generate a new dataset. The dataset corresponding to $k=0.5$ is referred to as the \textit{base} dataset.
 
Let $T^0(u)$ and $T^1(u)$ denote the temperature at time $u$ for the control and treatment groups, respectively. The potential outcomes $Y_0(u)$ and $Y_1(u)$ corresponding to the probability of suffering a burn at time $t$ for the control and treatment groups are described as follows:
$$
Y_j(u) = \max\left(\frac{1}{75}(T^j(u)-25),0\right)
$$

\paragraph{Movement (Physics)} Consider a free-falling object encountering air resistance. Opening the parachute can change the air resistance and control the descent velocity. The causal inference task of interest is the effect of the air resistance (e.g., with $a=1$ or without parachute $a=0$)  on the object's velocity at different times.

In this experiment, the family of datasets is generated, consisting of $12$ datasets. Each dataset includes $4000$ observations (e.g., $2000$ in the treatment group and $2000$ in the control group). The covariate is the time $u$. The outcome is the velocity at time $u$. The times of the treatment and control groups are sampled from two Chi-squared distributions $\chi^{2}(2)$ and $\chi^{2}(5)$, respectively, to create artificial bias. 

The underlying causal structure is governed by an ordinary differential equation (ODE) with the following analytical solution describing the velocity of a person at time $u$:
\begin{align}
\label{movement_equation}
    v(u) = \frac{g}{C}+(v(0)-\frac{g}{C})e^{-Cu}
\end{align}

where $g= 10$ is the earth's gravitational constant, $ C = k/m$, and $m, k$ are the mass and the air resistance constant, respectively. We assume that $v(0) = 0$ corresponds to a free-falling object without initial velocity.

For the control group, $m=k=C=1$ and the potential outcome is calculated as $Y_0(u)= v(u) = 10-e^{-u}$.  We use different sets of $(m,k)$ to generate the treatment groups for each dataset. The values of $(m,k)$ used in this experiment are as follows:
$(5,1)$, $(5,5)$, $(5,10)$, $(5,20)$, $(10,5)$, $(10,10)$, $(10,20)$, $(20,5)$, $(20,10)$, $(20,20)$, $(50,10)$, $(50,20)$.
The potential outcome function $Y_1 (u)$ is calculated from Equation~\ref{movement_equation} with the values of $m, k$ shown above.
We choose the dataset corresponding to $(m,k)=(5,1)$ as the \textit{base} dataset.

\subsubsection{Details of Experiments}
\label{exp_detail}

In this paper, we first create a number of causal inference tasks from the above families of datasets. For each family of datasets (e.g., IHDP, Jobs, Twins), the \textbf{base} task is created from its \textit{base} dataset. Similarly, we construct the other tasks from the remaining datasets in that family. In order to study the effects of transfer learning on causal inference, we define the source tasks and the target tasks as follows:
\begin{itemize}
    \item In the first experiment in Section ~\ref{exp1}, we choose the \textit{base} task to be the source task and the other tasks to be the target tasks.
    \item In the second experiment in Section ~\ref{exp2}, we choose the \textit{base} task to be the target task and the other tasks to be the source tasks.
\end{itemize}

\section{Proofs of Theorems}\label{proof}

\begin{theorem1}
Let $\hat{f}^{S}$ be a model trained on a source task, then
\begin{align*}
    \epsilon^{T}_{F}(\hat{f}^{S}) + u \epsilon^{T,a=0}_{CF}(\hat{f}^{S}) \leq \varepsilon_{PEHE}^{T}(\hat{f}^{S})
\end{align*}
where $u = p^{T}_{F}(a=1) $.
\end{theorem1}

\begin{proof}[\textbf{Proof of Theorem~\ref{thm:lb_epehe}}]
\label{pf:risk}

We have: 
\begin{align}\label{eq1}
\begin{aligned}
& \varepsilon_{PEHE}(\hat{f^{S}})\\ & = \int_{\mathcal{X}} \big[(\hat{f}^{S}(x,1)  - \hat{f}^{S}(x,0)) - (f^{T}(x,1) - f^{T}(x,0))\big]^{2}\\ & \quad p^{T}_F(x) dx \\
& = \int_{\mathcal{X}} \big[(\hat{f}^{S}(x,1)  - f^{T}(x,1)) - (f^{T}(x,0) - \hat{f}^{S}(x,0))\big]^{2} \\ & \quad p^{T}_F(x) dx \\
& = \int_{\mathcal{X}} (\hat{f}^{S}(x,1)  - f^{T}(x,1))^{2} p_{F}(x) dx \\
& \quad + \int_{\mathcal{X}} (\hat{f}^{S}(x,0)  - f^{T}(x,0))^{2} p^{T}_{F}(x) dx \\
& \quad - 2 \int_{\mathcal{X}} (\hat{f}^{S}(x,1)  - f^{T}(x,1))(f^{T}(x,0)- \hat{f}^{S}(x,0))\\
&\quad \; p^{T}_F(x) dx \\
\end{aligned}
\end{align}

First, we have the following properties of the factual and counterfactual distributions:
$$
\begin{aligned}
1. \; &\forall x \in \mathcal{X},\; p_F(x) = p_{CF}(x) \\
2. \; & \forall x \in \mathcal{X},\forall a \in \{0,1\},\; p_F(x,a) = p_{CF}(x,1-a)
\end{aligned}
$$

Applying these properties, the first term of Equation~(\ref{eq1}) can be expressed as:
$$
\begin{aligned}
& \int_{\mathcal{X}} (\hat{f}^{S}(x,0)  - f^{T}(x,0))^{2} p^{T}_{F}(x) dx \\
& =  u \int_{\mathcal{X}} (\hat{f}^{S}(x,0)  - f^{T}(x,0))^{2} p^{T}_{F}(x|a=1) dx \\
&\; + (1-u) \int_{\mathcal{X}} (\hat{f}^{S}(x,0)  - f^{T}(x,0))^{2} p^{T}_{F}(x|a=0) dx \\
& =  u \int_{\mathcal{X}} (\hat{f}^{S}(x,0)  - f^{T}(x,0))^{2} p^{T}_{CF}(x|a=0) dx \\
&\; + (1-u) \int_{\mathcal{X}} (\hat{f}^{S}(x,0)  - f^{T}(x,0))^{2} p^{T}_{F}(x|a=0) dx \\
& = u \epsilon^{T,a=0}_{CF}(\hat{f}^{S}) + (1-u)\; \epsilon^{T,a=0}_{F}(\hat{f}^{S}) 
\end{aligned}
$$
Similarly, the second term of Equation~(\ref{eq1}) can be expressed as:
$$
\begin{aligned}
& \int_{\mathcal{X}} (\hat{f}^{S}(x,1)  - f^{T}(x,1))^{2} p^{T}_{F}(x) dx \\
& =  (1-u) \epsilon^{T,a=1}_{CF}(\hat{f}^{S}) + u\; \epsilon^{T,a=1}_{F}(\hat{f}^{S})
\end{aligned}
$$

The potential outcome is independent given the features $Y_1 \ind Y_0 | X$ due to its unconfoundedness. Hence, the third term of Equation~(\ref{eq1}) can be expressed as:
$$
\begin{aligned}
    & \mathbb{E}\big[(\hat{f}^{S}(X,1)  - f^{T}(X,1))(f^{T}(X,0)- \hat{f}^{S}(X,0))\big] \\
    & = \mathbb{E}_x\Bigg[\mathbb{E}\Big[\hat{f}^{S}(x,1)  - Y^{T}_1)(Y^{T}_0- \hat{f}^{S}(x,0))|X=x\Big]\Bigg] \\
    & = 0
\end{aligned}
$$

The factual and counterfactual losses of the treatment and control groups  are positive. Thus, we have:

\begin{align*}
\begin{aligned}
    & u \epsilon^{T,a=1}_{F}(\hat{f}^{S}) + (1-u) \epsilon^{T,a=0}_{F}(\hat{f}^{S}) + u \epsilon^{T,a=0}_{CF}(\hat{f}^{S})\\
    & = \epsilon^{T}_{F}(\hat{f}^{S}) + u \epsilon^{T,a=0}_{CF}(\hat{f}^{S}) \\
    & \leq \varepsilon_{PEHE}^{T}(\hat{f}^{S})
\end{aligned}
\end{align*}
\end{proof}

\begin{theorem2}
For any hypothesis $\hat{f}$, we have:
\begin{align}
    \begin{aligned}
        \epsilon^{T}_{CF}(\hat{f}) \leq & \epsilon^{S}_{F}(\hat{f}) + 
        V(p^{T}_{F},p^{S}_{F}) + V(p^{T}_{F},p^{T}_{CF}) \\ & + \mathbb{E}_{p^{S}_{F}}[|f^{S}(x,t) - f^{T}(x,t)|]  
    \end{aligned}
\end{align}
and
\begin{align}
    \begin{aligned}
        \varepsilon^{T}_{PEHE}(\hat{f}) \leq & 4 \epsilon^{S}_{F}(\hat{f}) + 
        4 V( p^{T}_{F},p^{S}_{F}) + 2 V(  p^{T}_{F},p^{T}_{CF}) \\ & + 4 \mathbb{E}_{p^{S}_{F}}[|f^{S}(x,a) - f^{T}(x,a)|]  
    \end{aligned}
\end{align}
\end{theorem2}

\begin{proof}[\textbf{Proof of Theorem~\ref{thm:smiple_d1}}]
\label{pf:simple_td}
Adapting the first theorem in \cite{Ben-David2010} to our setting, we have the following two inequalities: 
$$
\epsilon^{T}_{CF}(\hat{f}) \leq \epsilon^{T}_{F}(\hat{f}) + V(p^{T}_{F},p^{T}_{CF})  
$$
and
$$
\epsilon^{T}_{F}(\hat{f}) \leq \epsilon^{S}_{F}(\hat{f}) + V(p^{T}_{F},p^{S}_{F}) + \mathbb{E}_{p^{S}_{F}}[|f^{S}(x,a) - f^{T}(x,a)|]   
$$
Therefore, we have:
$$
    \begin{aligned}
        \epsilon^{T}_{CF}(\hat{f}) \leq & \epsilon^{S}_{F}(\hat{f}) + 
        V(p^{T}_{F},p^{S}_{F}) + V(p^{T}_{F},p^{T}_{CF}) \\ & + \mathbb{E}_{p^{S}_{F}}[|f^{S}(x,a) - f^{T}(x,a)|]  
    \end{aligned}
$$
From \cite{shalit}, we have: 
$$
\varepsilon_{PEHE}^{T}(\hat{f}) \leq 2 \epsilon_F^{T}(\hat{f}) + 2 \epsilon_{CF}^{T}(\hat{f})
$$

Therefore, we have:
$$
\begin{aligned}
        \varepsilon^{T}_{PEHE}(\hat{f}) \leq & 4 \epsilon^{S}_{F}(\hat{f}) + 
        4 V( p^{T}_{F},p^{S}_{F}) + 2 V(  p^{T}_{F},p^{T}_{CF}) \\ & + 4 \mathbb{E}_{p^{S}_{F}}[|f^{S}(x,a) - f^{T}(x,a)|]  
    \end{aligned}
$$
\end{proof}

\begin{theorem3}
Suppose that the function class $G$ is stable under addition and multiplication and $\hat f, f^{T} \in G$, then
\begin{align}
    \begin{aligned}
        \epsilon^{T}_{CF}(\hat{f}) \leq & \epsilon^{S}_{F}(\hat{f}) + 
        \underset{G}{\text{IPM}}(p^{T}_{F},p^{S}_{F}) + \underset{G}{\text{IPM}}(p^{T}_{F},p^{T}_{CF}) \\ & + \mathbb{E}_{p^{S}_{F}}[|f^{S}(x,a) - f^{T}(x,a)|]  
    \end{aligned}
\end{align}
and 
\begin{align}
    \begin{aligned}
        \varepsilon^{T}_{PEHE}(\hat{f}) \leq & 4 \epsilon^{S}_{F}(\hat{f}) + 
        4 \underset{G}{\text{IPM}}( p^{T}_{F},p^{S}_{F}) + 2 \underset{G}{\text{IPM}}(  p^{T}_{F},p^{T}_{CF}) \\ & + 4 \mathbb{E}_{p^{S}_{F}}[|f^{S}(x,a) - f^{T}(x,a)|]  
    \end{aligned}
\end{align}
\end{theorem3}

\begin{proof}[\textbf{Proof of Theorem~\ref{thm:simple_ipm}}]
we have that:
\begin{equation*}
\begin{aligned}
\epsilon^{T}_{CF}(\hat{f}) \leq & \; \epsilon^{T}_{F}(\hat{f}) + \|\int (f^{T}(x,a)-\hat{f}(x,a))^2\\ & \; (p_F^{T}(x,a) - p_{CF}^{T}(x,a)) da dx  \| \\
 \leq & \epsilon^{T}_{F}(\hat{f}) + \underset{g\in G}{\sup}\|\int g(x,a)\\ & \quad \;(p_F^{T}(x,a) - p_{CF}^{T}(x,a)) da dx  \| 
\end{aligned}
\end{equation*}

Hence, we have:
$$
\begin{aligned}
\epsilon^{T}_{CF}(\hat{f}) \leq \epsilon^{T}_{F}(\hat{f}) + \underset{G}{\text{IPM}}(  p^{T}_{F},p^{T}_{CF})  
\end{aligned}
$$

Similarly, we have:
$$
\begin{aligned}
    & \epsilon^{T}_{F}(\hat{f})\\& \leq \epsilon^{S}_{F}(\hat{f}) + \mathbb{E}_{p^{S}_{F}}[|f^{S}(x,a) - f^{T}(x,a)|] \\ 
    & + \|\int (f^{S}(x,a)-\hat{f}(x,a))^2(p_F^{S}(x,a) - p_{F}^{S}(x,a)) da dx  \|  \\
    & \leq \epsilon^{T}_{F}(\hat{f}) + \mathbb{E}_{p^{S}_{F}}[|f^{S}(x,a) - f^{T}(x,a)|] + \underset{G}{\text{IPM}}( p^{T}_{F},p^{S}_{F})
\end{aligned}
$$

Thus, we have:
$$
\begin{aligned}
&\epsilon^{T}_{F}(\hat{f}) \\
&\leq \epsilon^{S}_{F}(\hat{f}) + \mathbb{E}_{p^{S}_{F}}[|f^{S}(x,a) - f^{T}(x,a)|] + \underset{G}{\text{IPM}}( p^{T}_{F},p^{S}_{F})
\end{aligned}
$$
Therefore, we have:
$$
    \begin{aligned}
        \epsilon^{T}_{CF}(\hat{f}) \leq & \epsilon^{S}_{F}(\hat{f}) + 
        \underset{G}{\text{IPM}}(p^{T}_{F},p^{S}_{F}) + \underset{G}{\text{IPM}}(p^{T}_{F},p^{T}_{CF}) \\ & + \mathbb{E}_{p^{S}_{F}}[|f^{S}(x,a) - f^{T}(x,a)|]  
    \end{aligned}
$$
From \cite{shalit}, we have: 
$$
\varepsilon_{PEHE}^{T}(\hat{f}) \leq 2 \epsilon_F^{T}(\hat{f}) + 2 \epsilon_{CF}^{T}(\hat{f})
$$

Therefore, we have:
$$
\begin{aligned}
        \varepsilon^{T}_{PEHE}(\hat{f}) \leq & 4 \epsilon^{S}_{F}(\hat{f}) + 
        4 \underset{G}{\text{IPM}}( p^{T}_{F},p^{S}_{F}) + 2 \underset{G}{\text{IPM}}(  p^{T}_{F},p^{T}_{CF}) \\ & + 4 \mathbb{E}_{p^{S}_{F}}[|f^{S}(x,a) - f^{T}(x,a)|]  
    \end{aligned}
$$
\end{proof}

Next, we will use the following results from ~\cite{shalit} for causal inference. For $x\in \mathcal{X}, a\in\{0,1\}$, with notation simplicity, we define:
$$
L_{\Phi,h}^{T}(x,a) = \int_{Y}l_{\Phi,h}(x,a,y)P(Y^{T}_a = y|x)dy.
$$

\begin{theorem}[Bounding The Counterfactual Loss]
\label{bounding_cfl}
Let $\Phi$ be an invertible representation with inverse $\Psi$. 
Let $p_{\Phi}^{a=i} = p_{\phi}(r|a=i),a\in\{0,1\}$
Let $h: \mathcal{R} \times\{0,1\} \rightarrow \mathcal{Y}$ be a hypothesis. Assume that for $a=0,1$, the function $r\mapsto L_{\Phi,h}(\Psi(r), a) \in G$ then:
\begin{align}
\begin{aligned}
&\epsilon_{C F}(\Phi,h) \leq \\
&(1-u) \epsilon_{F}^{a=1}(\Phi,h)+a \epsilon_{F}^{a=0}(\Phi,h)+ \\
& \underset{G}{\text{IPM}}\left(p_{\Phi}^{a=1}, p_{\Phi}^{a=0}\right).
\end{aligned}
\end{align}
\end{theorem}

\begin{theorem}[Bounding the $\epsilon_{PEHE}$]
\label{bounding_epehe}
The Expected Precision in Estimating Heterogeneous Treatment Effect $\epsilon_{PEHE}$ satisfies
\begin{align}
\begin{aligned}
&\varepsilon_{PEHE}(\Phi,h) \\
& \leq 2\left(\epsilon_{C F}(\Phi,h)+\epsilon_F(\Phi,h)\right)   \\
&\leq 2\left(\epsilon_F^{a=0}(\Phi,h)+\epsilon_F^{a=1}(\Phi,h)+\underset{G}{\text{IPM}}\left(p_{\Phi}^{a=1}, p_{\Phi}^{a=0}\right)\right)
\end{aligned}
\end{align}
\end{theorem}

In the next section, the performance of target task $\epsilon_F^{T,a=0}(\Phi,h)$ is related to that of a source task $\epsilon_F^{S,a=0}(\Phi,h)$. Without loss of generality, we present the proof for the case when $a=0$. %Proof of Lemma 1 requires the most work, Lemma 2 and Theorem 1 result from direct application of Lemma 1. 

First, we make the following assumptions:
\begin{itemize}
    \item \textbf{A1}: $\Phi$ is injective (Thus, $\Psi = \Phi^{-1}$ exists on $\text{Im}(\Phi)$).
    \item \textbf{A2}: There exists a real function space $G$ on $\text{Im}(\Phi)$ such that the function $r \mapsto \ell^{T}_{\Phi,h}(\Psi(r), a,y) \in G$.
    % \item\label{transferability_assumption}\textbf{Assumption 3}: \textbf{Causal Knowledge Transferability Assumption}: There exists a function class $G'$ on $\mathcal{Y}$ such that $y\mapsto \ell_{\Phi,h}(x,t,y) \in G'$ and  $\mathbb{E}\left[\underset{G'}{\text{IPM}}(P(Y_t^{S}|x), P(Y_t^{T}|x))\right]\le\delta$ for $t\in \{0,1\}$.
    \item \textbf{A3}: There exists a function class $G'$ on $\mathcal{Y}$ such that $y\mapsto \ell_{\Phi,h}(x,a,y) \in G'$. 
    %and almost surely on $\mathcal{X}$ with respect to $P(X^{Sr})$.
\end{itemize}
The measure of the fundamental difference between two causal inference tasks is defined as follows:
\begin{equation*}\label{transferability_assumption}
    \gamma^* = \mathbb{E}_{x \sim P(X^S)}\left[\underset{G'}{\text{IPM}}(P(Y_a^{S}|x), P(Y_a^{T}|x))\right]
\end{equation*}

\begin{lemma}
\label{lemma1}
Suppose that Assumptions 1-3 hold. The factual losses of any model $(\Phi,h)$ on source and target task satisfy for every $a \in \{0,1\}$
\begin{equation*}
    \begin{aligned}
        &\epsilon_{F}^{T,a}(\Phi,h) \le \\&\epsilon_{F}^{S,a}(\Phi,h) +  \underset{G}{\text{IPM}}(P(\Phi(X_a^{T})),P(\Phi(X_a^{S}))) + \gamma^{*} 
    \end{aligned}
\end{equation*}
\end{lemma}

\begin{proof}[\textbf{Proof of Lemma~\ref{lemma1}}] 
\begin{align*}
\begin{aligned}
&\epsilon_{F}^{T,a=0}(\Phi,h) - \epsilon_{F}^{S,a=0}(\Phi,h) \\
& = \int_{\mathcal{X}} L_{\Phi,h}^{T}(x,0)P(X_0^{T}=x)-L_{\Phi,h}^{S}(x,0)P(X_0^{S}=x)dx \\
& = \int_{\mathcal{X}} L_{\Phi,h}^{T}(x,0)P(X_0^{T}=x)-L_{\Phi,h}^{T}(x,0)P(X_0^{S}=x)\\ & + L_{\Phi,h}^{T}(x,0)P(X_0^{S}=x) - L_{\Phi,h}^{S}(x,0)P(X_0^{S}=x)dx \\
& = \underbrace{\int_{\mathcal{X}} L_{\Phi,h}^{T}(x,0)P(X_0^{T}=x)-L_{\Phi,h}^{T}(x,0)P(X_{0}^{S}=x)dx}_{\mathlarger{\Gamma}}\\
&+\underbrace{\int_{\mathcal{X}}\left(L_{\Phi,h}^{T}(x,0)-L_{\Phi,h}^{S}(x,0)\right)P(X_0^{S}=x)dx}_{\mathlarger{\Theta}}
\end{aligned}
\end{align*}

To bound $\Theta$, we use the following inequality:
\begin{align*}
\begin{aligned}
& L_{\Phi,h}^{T}(x,t)-L_{\Phi,h}^{S}(x,t) \\
&= \int_{Y}\ell_{\Phi,h}(x,a,y)\left(P(Y^{T}_a = y|x)-P(Y^{S}_a = y|x)\right)dy \\
&\le \max_{f \in G'}\Bigg|\int_{Y}f(y)P(Y^{T}_a = y|x)-P(Y^{S}_a = y|x)dy\Bigg| \\
&= \underset{G'}{\text{IPM}}\big(P(Y^{T}_a = y|x), P(Y^{S}_a = y|x)\big) 
\end{aligned}
\end{align*}

From the above inequality, we have:
\begin{align*}
\begin{aligned}
\Theta &= \int_{\mathcal{X}}\left(L_{\Phi,h}^{T}(x,0)-L_{\Phi,h}^{S}(x,0)\right)P(X_0^{S}=x)dx \\
& \leq \mathbb{E}_{x \sim P(X^S)}\left[\underset{G'}{\text{IPM}}(P(Y_a^{S}|x), P(Y_a^{T}|x))\right]\\
& = \gamma^{*}
\end{aligned}
\end{align*}

To bound $\Gamma$, we use the change of variable formula:
$$
\begin{aligned}
 \Gamma  &= \int_{\mathcal{X}} L_{\Phi,h}^{T}(x,0)P(X_0^{T}=x) - \\
& \quad L_{\Phi,h}^{T}(x,0)P(X_0^{S}=x)dx \\
&= \int_{\mathcal{R}} L_{\Phi,h}^{T}\big(\Psi(r),0\big)P\big(\Phi(X^{T}_{0})=r\big) - \\
& \quad L_{\Phi,h}^{T}\big(\Psi(r),0\big)P\big(\Phi(X^{S}_{0}) =r\big)dr\\
&\le \max_{g\in G} \Bigg|\int g(r) \Big(P\big(\Phi(X^{T}_{0})=r\big)-\\
& \quad P\big(\Phi(X^{S}_{0}) =r\big)\Big)dr\Bigg|\\
&= \underset{G}{\text{IPM}}\Big(P\big(\Phi(X^{T}_0)\big),P\big(\Phi(X^{S}_0\big)\Big)
\end{aligned}
$$
Combining the above upper bounds for $\Gamma$ and $\Theta$, we have:
$$
\begin{aligned}
&\epsilon_{F}^{T,a=0}(\Phi,h) - \epsilon_{F}^{S,a=0}(\Phi,h) \\
& \le  \underset{G}{\text{IPM}}\Big(P\big(\Phi(X^{T}_0)\big),P\big(\Phi(X^{S}_0)\big)\Big) + \gamma^{*} 
\end{aligned}
$$

Thus, we conclude that:
\begin{equation*}
\begin{aligned}
&\epsilon_{F}^{T,a=0}(\Phi,h)\\
&\le \epsilon_{F}^{S,a=0}(\Phi,h) + \underset{G}{\text{IPM}}\Big(P\big(\Phi(X^{T}_0)\big),P\big(\Phi(X^{S}_0)\big)\Big) + \gamma^{*}
\end{aligned}
\end{equation*}
\end{proof}

\begin{lemma2}
Suppose that Assumptions A1, A2, A3 hold. Then the counterfactual loss of any model $(\Phi,h)$ on the target task satisfy:
\begin{equation*}
    \begin{aligned}
        \epsilon_{CF}^{T}(\Phi,h) \le &\epsilon_F^{S,a=1}(\Phi,h) + \epsilon_F^{S,a=0}(\Phi,h)\\ 
                               & + \underset{G}{\text{IPM}}(P(\Phi(X_1^{T})),P(\Phi(X_1^{S}))) \\
                               & + \underset{G}{\text{IPM}}(P(\Phi(X_0^{T})),P(\Phi(X_0^{S}))) \\
                               & + \underset{G}{\text{IPM}}(P(\Phi(X_0^{T})),P(\Phi(X_1^{T})))+2\gamma^*\\
    \end{aligned}
\end{equation*}
where 
\begin{equation}
%\label{transferability_assumption}
    \gamma^* = \underset{{x \sim P(X^S)}}{\mathbb{E}}\left[\underset{G'}{\text{IPM}}(P(Y_a^{S}|x), P(Y_a^{T}|x))\right]
\end{equation}
measures the fundamental difference between two causal inference tasks.
\end{lemma2}

\begin{proof}[\textbf{Proof of Lemma~\ref{lemma2}}] 
Theorem \ref{bounding_cfl} is applied to establish an upper bound for the counterfactual loss of the target task. Subsequently, we apply Lemma~\ref{lemma1}.

$$
\begin{aligned}
&\epsilon^{T}_{CF}(\Phi,h) \\
& \leq \epsilon_{F}^{T,a=1}(\Phi,h)+\epsilon_{F}^{T, a=0}(\Phi,h)+ \underset{G}{\text{IPM}}\big(\Phi(X^{T}_0), \Phi(X^{T}_1)\big)
\end{aligned}
$$
Therefore,
\begin{equation*}
\begin{aligned}
\epsilon^{T}_{CF}(\Phi,h)&\leq \epsilon_F^{S,a=1}(\Phi,h) + \epsilon_F^{S,a=0}(\Phi,h) +2\gamma^{*} \\ & 
+\underset{G}{\text{IPM}} \Big(P\big(\Phi(X_1^{T})\big), P\big(\Phi(X_1^{S})\big)\Big) \\
&+ \underset{G}{\text{IPM}}\Big(P\big(\Phi(X_0^{T})\big),P\big(\Phi(X_0^{S})\big)\Big) \\
&+ \underset{G}{\text{IPM}}\Big(P\big(\Phi(X_0^{T})\big),P\big(\Phi(X_1^{T})\big)\Big)\\
\end{aligned}
\end{equation*}
\end{proof}

\begin{theorem5}{(Transferability of Causal Knowledge)} 
Suppose that Assumptions A1, A2, A3 hold. The performance of source model on target task, i.e. $\varepsilon^{T}_{PEHE}(\Phi,h)$, is upper bounded by:
\begin{equation*}
\begin{aligned}
    \varepsilon^{T}_{PEHE}(\Phi, h) \le &2(\epsilon_F^{S,a=1}(\Phi,h) + \epsilon_F^{S,a=0}(\Phi,h)\\ 
    &+\underset{G}{\text{IPM}}(P(\Phi(X_1^{T})),P(\Phi(X_1^{S})))\\
    & + \underset{G}{\text{IPM}}(P(\Phi(X_0^{T})),P(\Phi(X_0^{S}))) \\
    & + \underset{G}{\text{IPM}}(P(\Phi(X_0^{T})),P(\Phi(X_1^{T}))+2\gamma^*) 
\end{aligned}
\end{equation*}
\end{theorem5}

\begin{proof}[\textbf{Proof of Theorem~\ref{main_theorem_1}}]
By applying Theorem \ref{bounding_epehe}, we get 
\begin{align*}
\begin{aligned}
&\varepsilon^{T}_{P E H E}(\Phi,h) \\
&\leq 
2\Big(\epsilon_F^{T,a=0}(\Phi,h) + \epsilon_F^{T,a=1}(\Phi,h) \\&+ \underset{G}{\text{IPM}}\left(P\left(\Phi(X^{T}_0)\right), P\left(\Phi(X^{T}_1)\right)\right)\Big) \\
\end{aligned}
\end{align*}
After applying Lemma~\ref{lemma1} to the first and second terms of the above equation, we have:
\begin{equation*}
\begin{aligned}
    \varepsilon^{T}_{PEHE}(\Phi, h) \le & \; 2 \; (\epsilon_F^{S,a=1}(\Phi,h) + \epsilon_F^{S,a=0}(\Phi,h)\\ 
    &+\underset{G}{\text{IPM}}(P(\Phi(X_1^{T})),P(\Phi(X_1^{S})))\\
    & + \underset{G}{\text{IPM}}(P(\Phi(X_0^{T})),P(\Phi(X_0^{S}))) \\
    & + \underset{G}{\text{IPM}}(P(\Phi(X_0^{T})),P(\Phi(X_1^{T}))+2\gamma^*) 
\end{aligned}
\end{equation*}
\end{proof}

\section{Baseline: Data Bundling} 
In many causal inference scenarios, we only have access to the trained model, and the corresponding data is unavailable. This situation could be the case in medical applications due to privacy reasons. Consequently, bundling the datasets of source tasks with the target task is not feasible. In contrast, the data may be available for some specific applications. In this case, we create another baseline referred to as data bundling.

In data bundling, we create the bundled dataset by combining the datasets of source tasks and the target task. Here, we compare our approach with data bundling for the IHDP and the Movement(Physics) datasets. For data bundling, we report the model's best performance (i.e., $\varepsilon_{P E H E}$) achieved by hyper-parameter search. For our approach, we only report the model's performance with the lowest training error. This setup gives more advantage to the data bundling baseline. The results are illustrated in Figure~\ref{fig:data_bundling}. Even with the aforementioned advantage, the data bundling method achieves poorer performance than our approach. This is due to data imbalance, lack of precision in determining similarity from propensity score, and \textbf{differences in outcome functions}.

\begin{figure}[t]
\centering
    \centering
    \includegraphics[width=0.45\textwidth]{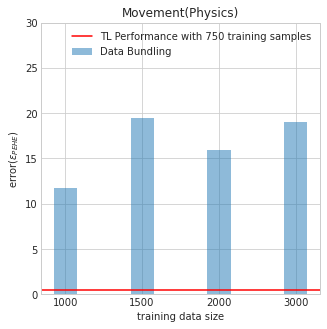}
    \centering
    \includegraphics[width=0.45\textwidth]{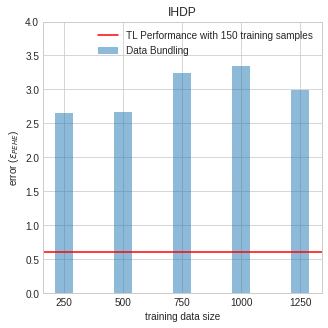}
    \caption{Performance comparison between data bundling and our approach. Our approach (red horizontal line) significantly outperforms data bundling. An increase in the size of training data doesn't improve the performance of data bundling.}
    \label{fig:data_bundling}
\end{figure}

\section{Causal Inference Task Affinity}
\label{epsnn}
Let $\mathcal{P}_{N_{\theta}}(T, D^{te})\in [0,1]$ be a function that measures the performance of a given model $N_{\theta}$ parameterized by $\theta\in\mathbb{R}^d$ on the test set $D^{te}$ of the causal task $T$.

\begin{definition}[$\varepsilon$-approximation Network]
A model $N_{\theta}$ is called an $\varepsilon$-approximation network for a task-dataset pair $(T,D)$ if it is trained using the training data $D^{tr}$ such that $\mathcal{P}_{N_{\theta}}(T, D^{te}) \geq 1 - \varepsilon$, for a given $0 < \varepsilon < 1$. 
\end{definition}

\begin{definition}[Fisher Information Matrix]
For a neural network $N_{\theta_{s}}$ with weights $\theta_{s}$ trained on data $D_{s}$, a given test dataset $D_t$ and the negative log-likelihood loss function $L(\theta,D)$, the Fisher Information matrix is defined as:
\begin{align}
    F_{s,t} &=\mathbb{E}_{D\sim D_t}\Big[\nabla_{\theta} L(\theta_{s},D)\nabla_{\theta} L(\theta_{s},D)^T\Big] \\
    &= -\mathbb{E}_{D\sim D_{t}}\Big[\mathbf{H}\big(L(\theta_{s},D)\big)\Big],
\end{align}
\end{definition} 
where $\mathbf{H}$ is the Hessian matrix, i.e., $\mathbf{H}\big(L(\theta,D)\big)= \nabla_{\theta}^2L(\theta,D)$, and expectation is taken w.r.t the data. It is proven that the Fisher Information Matrix is asymptotically well-defined \citep{9766163}.
In practice, we approximate the above with the empirical Fisher Information matrix:
\begin{align}\label{emprical_fisher}
    \hat{F}_{s,t} = \frac{1}{|D_{t}|}\sum_{x\in D_{t}} \nabla_{\theta} L(\theta_{s},x)\nabla_{\theta} L(\theta_{s},x)^T.
\end{align}
Here, the empirical Fisher Information Matrix is positive semi-definite because it is the summation of positive semi-definite terms, regardless of the number of samples.

% \subsubsection{Comparison between Unsymmetrized task affinity and CITA}
% We compare the nonsymmetrized task affinity~\citep{le2022task} and symmetrized task affinity (CITA) on the Jobs and the Twins datasets. Figure~\ref{fig:td_sym_plot} shows that CITA has successfully captured the symmetries within causal inference tasks. The x-axis $p$  denotes the probability of flipping treatment assignments of the original dataset. The proposed symmetrized task affinity shows that the datasets corresponding to $p=1$ (i.e., the flipped treatments dataset) and $p=0$ (i.e., the original dataset) are the closest tasks to the original task. The dataset with $p=0.5$ is the furthest dataset. We observe that the computed symmetrized task affinitys (CITA) have the y-axis symmetry at $p=0.5$, indicating the symmetry of the causal inference tasks.  In contrast, the nonsymmetrized task affinity fails to capture this symmetry property.

\subsection{Task Affinity Between Counterfactual Tasks}
\label{task_cf}
In the following section, we denote the task-dataset pair $a=(T_{a}, D_{a})$  by $a_{F}=(T_{a_{F}}, D_{a_{F}})$ where $D_{a_F}$ is sampled from the factual distribution. Similarly, $a_{CF} = (T_{a_{CF}}, D_{a_{CF}})$ denotes the counterfactual task-dataset pair, where $D_{a_{CF}}$ is sampled from   the counterfactual distribution. We refer to $(T_{a_{F}}, D_{a_{F}})$ and $(T_{a_{CF}}, D_{a_{CF}})$ as the corresponding factual and counterfactual tasks. 

The following theorem proves that the order of proximity of tasks is preserved even if we observe the counterfactual tasks instead. In other words, a task, which is more similar to the target task when measured using factual data, remains more similar to the target task even when measured using counterfactual data. 
% To ensure good performance on both the target factual and counterfactual distributions, we need both distances between factual and counterfactual tasks to be small. 
%Our theorem states that a small factual distance implies a small counterfactual distance. Hence, the causal knowledge can be transferred efficiently.

%then, we have all the sufficient conditions for good causal knowledge transfer performance. 

\begin{theorem}
\label{counterctual_order}
Let $\mathbb{T}$ be the set of tasks and
let $a_{F} = (T_{a_{F}},D_{a_{F}})$, $b_{F} = (T_{b_{F}},D_{b_{F}})$, and $c_{F} = (T_{c_{F}},D_{c_{F}})$ be three factual tasks and  $a_{CF} = (T_{a_{CF}},D_{a_{CF}})$, $b_{CF} = (T_{b_{CF}},D_{b_{CF}})$, and $c_{CF} = (T_{c_{CF}},D_{c_{CF}})$ their corresponding counterfactual tasks. 

Suppose that there exists a class of neural networks (well-trained causal inference neural networks) 
$\mathcal{N} = \{N_{\theta}\}_{\theta \in \Theta}$ for which:
\begin{align}
    \forall a,b,c \in \mathbb{T}, \: d[a,b] \leq d[a,c] + d[c,b]
\end{align}
and the task affinity between the factual and the counterfactual can be arbitrarily small, described as follows:
\begin{align}    
    \forall \epsilon>0, \exists N_{\theta} \in \mathcal{N}, \; d[a_{F},a_{CF}]<\epsilon
\end{align}

We have the following result:
\begin{align}
    d[a_{F},b_{F}]\leq d[a_{F},c_{F}] \implies d[a_{CF},b_{CF}] \leq d[a_{CF},c_{CF}]
\end{align}
\end{theorem}

\begin{proof}[\textbf{Proof of Theorem ~\ref{counterctual_order}}]
Suppose $d[a_{F},b_{F}]\leq d[a_{F},c_{F}]$. For every $\epsilon >0$, we have:
\begin{align*}
\begin{aligned}
d[a_{CF},b_{CF}] & \leq d[a_{CF},a_{F}] + d[a_{F},b_{F}] + d[b_{F},b_{CF}]\\
             & \leq \epsilon + d[a_{F},c_{F}] + \epsilon\\
             & \leq d[a_{F},a_{CF}] + d[a_{CF},c_{CF}] + d[c_{F},c_{CF}] \\ & \; +2\epsilon \\
             &\leq d[a_{CF},c_{CF}] + 4\epsilon
\end{aligned}
\end{align*}
Therefore,
$d[a_{CF},b_{CF}]\leq d[a_{CF},c_{CF}]$ as $\epsilon \to 0$.
\end{proof}

%\newpage
%\bibliography{aloui_189}
\end{document}